\documentclass{article}

\usepackage{microtype}
\usepackage{graphicx}
\usepackage{subcaption}
\usepackage{booktabs} % for professional tables

\usepackage{hyperref}

\usepackage[preprint]{icml2026}

\usepackage{amsmath}
\usepackage{amssymb}
\usepackage{mathtools}
\usepackage{amsthm}

\usepackage{bm}
\usepackage{physics}
\usepackage[utf8]{inputenc} % allow utf-8 input
\usepackage[T1]{fontenc}    % use 8-bit T1 fonts
\usepackage{booktabs}       % professional-quality tables
\usepackage{amsfonts}       % blackboard math symbols
\usepackage{nicefrac}       % compact symbols for 1/2, etc.
\usepackage{microtype}      % microtypography
\usepackage{lipsum}
\usepackage{fancyhdr}       % header
\usepackage{graphicx}       % graphics

\usepackage[capitalize,noabbrev]{cleveref}

\theoremstyle{plain}
\newtheorem{theorem}{Theorem}[section]

\newtheorem{lemma}[theorem]{Lemma}

\theoremstyle{definition}
\newtheorem{definition}[theorem]{Definition}
\newtheorem{assumption}[theorem]{Assumption}
\theoremstyle{remark}

\newtheorem{assertion}[theorem]{Assertion}

\usepackage[textsize=tiny]{todonotes}

\icmltitlerunning{Why Does Reasoning Length Converge? Unveiling the Underfitting-Overfitting Trade-off in Chain-of-Thought}

\begin{document}

\twocolumn[
  \icmltitle{Why Does Reasoning Length Converge? \\ Unveiling the Underfitting-Overfitting Trade-off in Chain-of-Thought}



  \icmlsetsymbol{equal}{*}

  \begin{icmlauthorlist}
  \icmlauthor{Zeyu Gan}{ruc}
  \icmlauthor{Yi Hao}{ruc}
  \icmlauthor{Yong Liu}{ruc}
  \end{icmlauthorlist}

  \icmlaffiliation{ruc}{Gaoling School of Artificial Intelligence, Renmin University of China, Beijing, China}

\icmlcorrespondingauthor{Yong Liu}{liuyonggsai@ruc.edu.cn}

  \icmlkeywords{test-time scaling, inference, large language models}

  \vskip 0.3in
]



\printAffiliationsAndNotice{}  % no special notice (required even if empty)

\begin{abstract}
Test-time scaling, primarily manifested through multi-step Chain-of-Thought (CoT) reasoning via Reinforcement Learning (RL), has emerged as a pivotal paradigm for enhancing the reasoning capabilities of Large Language Models (LLMs). However, a significant theoretical gap persists: traditional token-level analysis fails to capture the macroscopic dynamics of reasoning-level scaling. To address this, we introduce CoT-Space, a novel theoretical framework that recasts the reasoning process from a discrete token-prediction task to an optimization process within a continuous, reasoning-level semantic space. 
By modeling the reasoning trajectory from both noise and risk perspectives and revitalizing foundational principles from classical learning theory, we demonstrate that the observed convergence to an optimal CoT length is a natural consequence of the fundamental trade-off between underfitting and overfitting. We further utilize RL as a tool to elicit and verify these results in our experiments. Our findings provide a mechanistic explanation for the internal test-time scaling via RL, offering a principled theoretical foundation to optimize reasoning trajectories in modern LLMs. 
We open-source our code at \url{https://github.com/ZyGan1999/CoT-Space}. 
\end{abstract}

% MainTex
\section{Introduction}
\label{sec:introduction}

Improving the reasoning capabilities of Large Language Models (LLMs) is pivotal in modern AI research. The most well-known framework for this is Chain-of-Thought (CoT)~\citep{wei2022chain}, which elicits multi-step reasoning. A promising strategy to enhance CoT is ``test-time scaling''~\citep{chen2025empirical,jiang2024enhancing}, which can be broadly categorized into external and internal methods. External methods, such as tree-based strategies~\citep{yao2023tree,zhang2024rest,wan2024alphazero} and self-consistency~\citep{wang2022self}, augment the model's inference-time output %without additional training
via searching~\citep{gan2025rethinkingexternalslowthinkingsnowball}. %In contrast, internal methods embed slow-thinking %capabilities 
%directly into the model via post-training, typically using supervised fine-tuning (SFT) or reinforcement learning (RL).
While internal test-time scaling methods embed the reasoning capability directly into the model's inference process via post-training, typically utilizing Reinforcement Learning (RL) with a high-level outcome reward. 

Recent breakthroughs from models like DeepSeek's R1~\citep{deepseekai2025deepseekr1incentivizingreasoningcapability}, OpenAI's o1 and o3~\citep{openai2024reasoning,openai2025introduce-o3}, and Qwen's QwQ~\citep{qwq-32b-preview} have spotlighted the power of internal test-time scaling. This paradigm, also known as \textit{zero}-like training, involves applying RL directly to a pre-trained model, bypassing supervised tuning. The empirical results are compelling, showing that a high-level objective can guide a policy to discover an optimal strategy on its own. This reveals an interesting insight: the \textit{zero}-like training process is actually guiding the model to find an optimal inference policy. 
%LLMs can learn sophisticated reasoning strategies through self-exploration. 
%This aligns with a core tenet of RL: a high-level objective can guide a policy to discover an optimal strategy on its own. 

\begin{figure*}[tp]
    \centering
    \includegraphics[width=\linewidth]{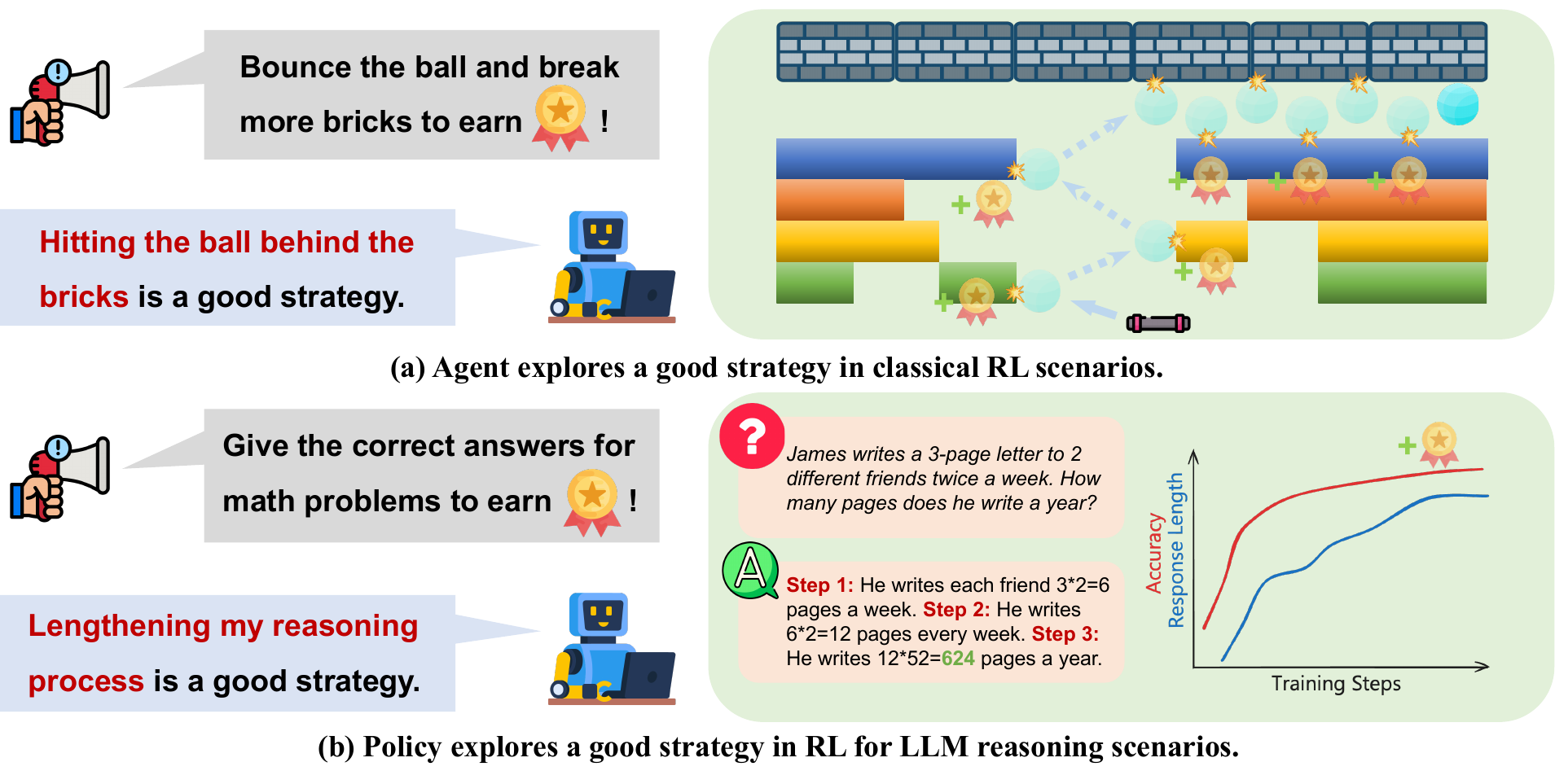}
    \caption{\textbf{Analogy between strategy discovery in classical RL and LLM reasoning.} (a) In classical RL, an agent with a high-level goal (e.g., break more bricks) discovers an effective strategy through exploration to maximize its reward. (b) Similarly, an LLM policy with the goal of providing correct answers autonomously learns that generating a suitable CoT is an effective strategy.} 
    %\caption{\textbf{An illustration of the analogy between strategy discovery in classical RL and RL for LLM reasoning.} (a) In a classical RL scenario such as a Bricks Breaker game, the agent is given a high-level goal to break more bricks. Through exploration, it discovers a non-trivial, effective strategy to hit the ball behind the bricks to maximize its reward. (b) In LLM reasoning, the policy model is given a high-level goal to provide the correct answers. As RL training progresses, the model autonomously discovers that generating a suitable CoT is an effective strategy to achieve the goal.}
    \label{fig:internal-slow-thinking}
    \vspace{-10pt}
\end{figure*}

As illustrated in Figure~\ref{fig:internal-slow-thinking}, \textit{zero}-like training for LLM reasoning is analogous to the classical strategy optimization process. 
In a classic RL task like Bricks Breaker\footnote{Bricks Breaker is a video game, the goal is to destroy the bricks by shooting a ball at them.} shown in subgraph (a), the agent's goal is to break more bricks. Without explicit instructions on specific strategies, the agent may discover a highly effective method, such as hitting the ball behind the top row of bricks to achieve a higher score. 
%This kind of emergent behavior is a common outcome in classical RL. 
Similarly, in the context of LLM reasoning, as shown in subgraph (b), the policy model is given the broad goal of providing a correct answer. As RL training progresses, the model converges on an optimal CoT length, demonstrating its ability to 
optimize via a simple, high-level reward signal. 

%Just as an agent in a game discovers an effective strategy to maximize its score, an LLM policy discovers that generating a coherent CoT is an effective strategy for solving problems and maximizing its reward.

A variety of RL-based methods have been developed to improve \textit{zero}-like post-training, including GRPO~\citep{shao2024deepseekmath}, VC-PPO~\citep{yuan2025s}, DAPO~\citep{yu2025dapo}, and VAPO~\citep{yue2025vapo}. However, despite these empirical successes, our theoretical understanding of internal test-time scaling remains shallow~\citep{gan2026beyond}. For instance, the convergence on an optimal CoT length during training, also known as the ``overthinking'' phenomenon~\citep{sui2025stop}, lacks a clear mechanistic explanation.

%``overthinking'' phenomenon~\citep{sui2025stop}, which suggests an optimal CoT length exists, lacks a clear mechanistic explanation. Furthermore, the principles governing the generalization of LLM reasoning are not yet well understood. 

This chasm between empirical success and theoretical understanding is alarming. It stands in stark contrast to old-school machine learning, where decades of research have forged a robust theoretical bedrock. Unfortunately, foundational principles %like optimization theory and generalization bounds, the very cornerstones of machine learning, 
appear to falter when applied to the complex, multi-step nature of LLM reasoning. This prompts the critical question that motivates our work: \textbf{Are these time-tested theories truly obsolete in this new era, or do we merely lack the conceptual framework to bridge them to the unique dynamics of LLM reasoning?} We argue for the latter, positing that by shifting our analytical perspective, we can revitalize these foundational theories and use them to build a principled understanding of LLM reasoning.

The central barrier to this goal is twofold. \textbf{First, a fundamental misalignment exists between token-level analytical frameworks and the reasoning-level nature of CoT}, where actions are complete thoughts rather than single tokens. \textbf{Second, a conceptual gap separates the discrete world of language from the continuous mathematics underpinning classic learning theory}, hindering the application of its powerful analytical tools. To resolve both challenges, we introduce \textbf{CoT-Space}, a novel reasoning-level theoretical framework for LLM reasoning via RL. Served as a conceptual framework, we aim to build a theoretical bridge that connects the empirical phenomena of LLM reasoning with foundational principles of classical machine learning. Our framework first defines a reasoning-level state space to align the analytical framework with CoT's structure. Subsequently, we prove this space converges to a continuous manifold, a key result that recasts reasoning as an optimization process. By leveraging this continuous perspective, we are then able to conduct analyses that explain the convergence of an optimal CoT length from both noise and risk perspectives. We reveal that this convergence is actually an underfitting-overfitting trade-off during CoT generations. 
%and the generalization of reasoning, revitalizing classical theory to address the key theoretical gaps previously identified. 

The remainder of this paper is structured as follows. In Section~\ref{sec:cot-space}, we provide a systematic analysis of the misalignment between the current token-level inference formulation and the reasoning-level nature of CoT generations. Subsequently, we introduce our reasoning-level theoretical framework, CoT-Space. We then leverage this framework in Section~\ref{sec:analysis} to analyze the convergence of CoT length in the reasoning process from the perspective of noise and risk, demonstrating its potential theoretical value. We perform empirical validation of our theoretical insights in Section~\ref{sec:experiments}. Following this, we briefly review related works in Section~\ref{sec:related-work} and finally conclude the paper in Section~\ref{sec:conclusion}.

\section{CoT-Space: A Reasoning-Level Theoretical Framework}
\label{sec:cot-space}
% This section presents a systematic analysis of the limitations inherent in the current token-level RL formulation when applied to %higher-order cognitive processes of 
% CoT reasoning. We first discuss the fundamental misalignment of this conventional approach in Subsection~\ref{subsec:from-token-to-reasoning} and then introduce our novel, reasoning-level theoretical framework CoT-Space in Subsection~\ref{subsec:formulation-of-cot-space}. Furthermore, we discuss the continuum nature of the reasoning-level state space in Subsection~\ref{subsec:on-the-continuum-nature-of-the-reasoning-level-state-space} to establish the theoretical foundation of CoT-Space. 

This section presents a systematic analysis of CoT reasoning. We first discuss existing misalignment (Subsection~\ref{subsec:from-token-to-reasoning}), then introduce our reasoning-level framework, CoT-Space (Subsection~\ref{subsec:formulation-of-cot-space}), and finally prove its continuum nature to establish the theoretical foundation (Subsection~\ref{subsec:on-the-continuum-nature-of-the-reasoning-level-state-space}). 

%discuss the continuum nature of the reasoning-level state space in Subsection~\ref{subsec:on-the-continuum-nature-of-the-reasoning-level-state-space}. Finally, we introduce our novel, reasoning-level theoretical framework CoT-Space in Subsection~\ref{subsec:formulation-of-cot-space}.  
\subsection{From Token-Level to Reasoning-Level Perspective}
\label{subsec:from-token-to-reasoning}
% The conventional theoretical framework for applying RL to LLMs is fundamentally misaligned with the reasoning-level nature of CoT scenarios. This token-level assumption is inherited from traditional RL frameworks, where the agent receives a reliable and meaningful reward at each state transition. For instance, classic RL tasks like chess, Go, or video games allow for a meaningful measurement of a state's value, which guides the agent's policy.

% However, in the context of natural language, the auto-regressive nature of modern language models has led us to formulate each token prediction as a state transition, resulting in a popular token-level theoretical framework. We briefly review this formulation here. Prior works usually model natural language generation as a token-level Markov Decision Process (MDP): $\mathcal{M}(\mathcal{S}, \mathcal{A}, r, H)$, where $S$ is the state space of existing token sequences, $A$ is the action space representing the token dictionary, $r$ is the reward function guiding generation, and $H$ is the token budget.
% A policy $\pi: \mathcal{S} \rightarrow \mathcal{A}$ produces the next token, and the state transitions by concatenating the predicted token $a=\pi(s)$ to the existing sequence $s$. This formulation presents a significant challenge regarding the state space $\mathcal{S}$ as the following assertion. 

The conventional framework for LLM reasoning is fundamentally misaligned with the reasoning-level nature of CoT scenarios. This misalignment stems from modeling auto-regressive language generation as a token-level Markov Decision Process (MDP)~\citep{setlur2025scaling}: $\mathcal{M}(\mathcal{S}, \mathcal{A}, r, H)$, where $\mathcal{S}$ is the state space of existing token sequences, $\mathcal{A}$ is the action space representing the token dictionary, $r$ is the reward function guiding generation, and $H$ is the token budget. In this standard formulation, a policy $\pi: \mathcal{S} \rightarrow \mathcal{A}$ produces the next token $a=\pi(s)$ to extend the current sequence $s$. While inherited from classic RL tasks, this formulation presents a significant challenge regarding the token-level state space $\mathcal{S}$ in natural language, as we assert next (details in Appendix~\ref{app:limitation-of-natural-language}):

\begin{assertion}\label{assertion:huge-and-discontinuous}
The token-level state space $\mathcal{S}$ of natural language exhibits \textbf{exponential growth} and \textbf{discrete topology}. 
\end{assertion} 

%Consider a CoT reasoning process comprises multiple conceptual steps $(\xi_1,\xi_2,\dots,\xi_L)$, which naturally reveals two distinct procedural layers:
%\textbf{(1) Reasoning-Level:} The strategic planning of abstract reasoning steps $(\xi_1 \rightarrow \xi_L)$.
%\textbf{(2) Token-Level:} The tactical realization of specific steps via token generation. Traditional MDP formulations, however, operate only at the problematic token level, failing to capture the higher-order planning inherent to CoT. 

%The token-level state space $\mathcal{S}$ is problematic for two reasons. First, it grows exponentially with sequence length ($\mathcal{S}_{t}=\mathcal{A}^{t}$). %making it too vast to explore effectively during training. 
%Second, the auto-regressive process imposes a discrete and unidirectional topology, %($\forall \bm{s}_t \in \mathcal{S}_t, \mathcal{T}(\bm{s}_t,\pi) \in \mathcal{S}_{t+1}$)
%further complicating analysis. 
%This narrow focus creates a theoretical barrier that obscures the problem's true nature. To establish a sound foundation, we must therefore transition from a token-level formulation to a reasoning-level perspective. 

To address this, we propose shifting the analytical perspective from the microscopic token level to the macroscopic reasoning level. We posit that while the token-level state space is discrete and sparse, the reasoning-level state space, which is formed by abstract thoughts, can be approximated as a continuous manifold. This is analogous to interstellar travel: while individual stars (tokens) are discrete, the galaxy (semantic space) appears continuous from a distance. This ``zoom-out'' perspective allows us to apply continuous optimization tools to analyze the reasoning process, as discussed in Figure~\ref{fig:space-by-diff-scale}.

\subsection{Formulation of CoT-Space}
\label{subsec:formulation-of-cot-space}

\begin{figure}[t]
    \centering
    \includegraphics[width=\linewidth]{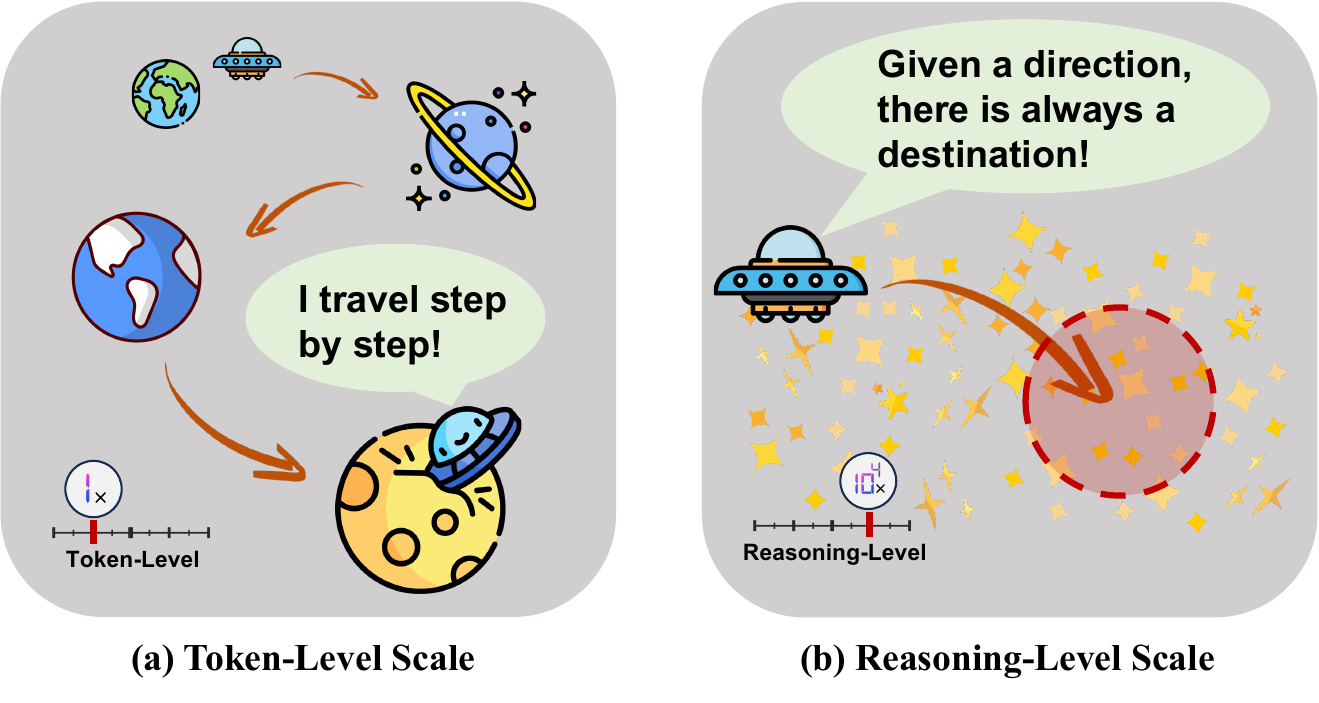}
    \caption{\textbf{Token-level vs. Reasoning-level perspectives.} (a) The token-level view treats the generation process as a path through discrete states. (b) The reasoning-level view zooms out, approximating the dense state space as continuous.}
    %\caption{\textbf{An illustration for the shift from a token-level to a reasoning-level perspective.} (a) The traditional token-level scale views the generation process as a sequence of discrete transitions in a vast state space. (b) The proposed reasoning-level scale zooms out to view the state space from a macroscopic perspective, where the incredibly dense and numerous token-level states can be approximated as a continuous semantic reasoning space.}
    \label{fig:space-by-diff-scale}
    \vspace{-10pt}
\end{figure}

%Building upon the discussion above, we now develop a reasoning-level theoretical framework designed to bridge the analytical gap. This alternative perspective draws an analogy to interstellar travel in space, as shown in Figure~\ref{fig:space-by-diff-scale}. From a token-level perspective, each individual token and its corresponding state are like discrete stars in a vast semantic galaxy. The auto-regressive generation process then resembles a journey across these stars. Since transitions only occur between states, the generation process is fundamentally discrete at the token-level scale, as illustrated in panel (a). 

%In contrast, from the reasoning-level perspective, a continuum approximation emerges when we ``zoom out'' to consider the limiting behavior of this discrete system. Due to the exponential density of semantically valid states, the global properties of the space converge toward continuum-like characteristics, as visualized in panel (b). This macroscopic viewpoint allows us to move beyond individual tokens and analyze reasoning from a more abstract, stepwise level. 

To formalize this conceptual ``zoom out'' framework, we define the following constructs. Note that these definitions rely on abstract semantic equivalence, serving as theoretical tools rather than computable functions. 

%To enable this higher-level analysis, we provide the following definitions for a new formalization of LLM reasoning via RL. %We begin by defining the reasoning-level state. 
Based on the CoT setting, where the output is generated step-by-step and consists of multiple intermediate thoughts, a reasoning state is defined as follows. 

\begin{definition}\label{def:reasoning-state}
    (CoT reasoning state.) Given a query $q$ and a set of intermediate abstract steps $\bm{\xi}_{o}=[\xi_1,\xi_2,\dots,\xi_t]$, the current reasoning state is defined as $s_o=(q,\bm{\xi}_o)$. 
\end{definition}

Definition~\ref{def:reasoning-state} establishes the concept of reasoning-level state. %a reasoning-level state that is a more meaningful unit of analysis than a single token. 
We then introduce a special class of states that represent successful problem solutions, which we call \textit{minimums}. 

\begin{definition}\label{def:minimums-in-reasoning-space}
    (Minimums in reasoning space.) Given a query $q$ with a definite golden answer $\phi$. Let $\Xi_q$ be the set of all reasonable intermediate reasoning processes for query $q$. Then the minimums of the reasoning space are defined as $\bm{M}_{q} = \left\{ (q,\bm{\xi},\phi)~|~\bm{\xi} \in \Xi_q \right\}$. 
\end{definition}

By Definition~\ref{def:minimums-in-reasoning-space}, minimums are states where the query has been successfully solved. For any intermediate reasoning state, we can then define which minimums are reachable. 

\begin{definition}\label{def:reachable-minimums}
    (Reachable minimums.) Given a query $q$ with a definite golden answer $\phi$, and an incomplete set of intermediate reasoning steps $\bm{\xi}_{o}=[\xi_1,\xi_2,\dots,\xi_t]$. The reachable minimums are defined as $\bm{R}^{o}_{q}=\left\{(q,\bm{\xi},\phi)~|~\bm{\xi}_{o} \sqsubseteq \bm{\xi}~\text{and}~\bm{\xi} \in \Xi_q \right\}$, where $a \sqsubseteq b$ represents that $a$ is prefix of $b$. 
\end{definition} 
Definition~\ref{def:reachable-minimums} establishes the concept of reachable minimums. This concept represents the set of successful reasoning paths that can be completed from the current state. We then quantify the distance between a state and its reachable minimums and define the nearest minimum. 

\begin{definition}\label{def:reachable-distance-and-nearest-reachable-minimum}
    (Reachable distance and nearest reachable minimum.) Given a query $q$ with a definite golden answer $\phi$, and a reasoning state $s_o$ on this query, the reachable distance between the $s_o$ and a reachable minimum is defined as $\operatorname{dist}(s_o, m)=|\bm{\xi}|-|\bm{\xi}_o|$, where $m \in \bm{R}^{o}_{q}$, $\bm{\xi}$ and $\bm{\xi}_o$ are intermediate steps of $m$ and $s_o$ respectively, $|\cdot|$ represents the length of a vector. The nearest minimum is defined as $m^{*}_{o}={\arg\min}_{m \in \bm{R}^{o}_{q}} \, \operatorname{dist}(s_{o},m)$. 
\end{definition}
Definition~\ref{def:reachable-distance-and-nearest-reachable-minimum} quantifies the number of reasoning steps required to reach the closest successful solution from a given state. Finally, we define a reasoning loss to characterize the gap between the current state and the optimal solution. 

\begin{definition}\label{def:reasoning-loss}
    (Reasoning loss.) Given a query $q$ with a definite golden answer $\phi$, and a reasoning state $s_o$ on this query, there exists a loss function $C$ that quantifies the distance of $s_o$ towards its nearest reachable minimum, which satisfies: for two different states $s_i, s_j$, $C(s_i,q) < C(s_j,q)$ if and only if $\operatorname{dist}(s_i,m^{*}_{i}) < \operatorname{dist}(s_j,m^{*}_{j})$. Specially, $\forall m \in \bm{M}_{q}, C(m,q)=0$. 
\end{definition}

Definition~\ref{def:reasoning-loss} conceptualizes the reasoning process as an optimization process. The function $C(\cdot,\cdot)$ acts as a proxy for the distance to a correct solution, where a lower loss indicates a more advanced stage of reasoning. 
Taking Figure~\ref{fig:cot-space}(a) as instance, in this transition, $m_i, m_j$ and $m_k$ are all reachable minimums for the current state $s_0$. However, the possible next states of $s_0$, which are denoted as $s_i, s_j$ and $s_k$ (Here we refer to the three realizations $s_{i_1}, s_{i_2}$ and $s_{i_3}$ all as $s_i$ for simplicity). These intermediate states differ in their distance to its corresponding minimums ($\operatorname{dist}(s_i, m_i) < \operatorname{dist}(s_j, m_j) < \operatorname{dist}(s_k, m_k)$). Given the same query $q$, the reasoning losses thus satisfy $C(s_i,q) < C(s_j,q) < C(s_k,q)$. 
%For instance, consider the problem in Figure~\ref{fig:internal-slow-thinking}. A state $s_i$ with the intermediate thought \textit{``He writes $12$ pages every week''} is closer to the final answer than a state $s_j$ with only \textit{``He writes each friend $3 \times 2=6$ pages a week.''} According to this definition, because $s_i$ requires fewer steps to reach a minimum (i.e., $\operatorname{dist}(s_i, m_i^*) < \operatorname{dist}(s_j, m_j^*)$), its reasoning loss must be lower, satisfying $C(s_i) < C(s_j)$. 

%With this complete CoT-Space formulation, we are able to analyze LLM reasoning at a higher level. In the next subsection, we subsequently show the continuum nature of the reasoning-level state space, thus yielding the rationality of CoT-Space. 

% \textcolor{blue}{Before proceeding, it is crucial to clarify the relationship between our theoretical loss $C(s,q)$ and the practical metrics used in our experiments. In our framework, $C(s,q)$ represents the ``distance to solution,'' where $C(m,q)=0$ for a correct answer. In our RL experiments, this corresponds to maximizing a task-level reward (e.g., $+1$ for success). Consequently, minimizing the expected reasoning loss is mathematically equivalent to maximizing the success rate (Accuracy), which serves as our primary evaluation metric. Although the theoretical loss landscape is continuous, the RL agent navigates it by learning a value function that approximates the expected future reward, effectively performing a stochastic descent on this unobserved loss surface to improve accuracy.} 

With this complete reasoning-level analytical framework established, we resolve the misalignment between token-level RL and the reasoning-level CoT paradigm. 
%\textcolor{blue}{Consider the math problem in Figure~\ref{fig:internal-slow-thinking} ($q$: "James writes..."), we here present an example for better clarification. State ($s_t$): The model has generated "Step 1: He writes 6 pages a week. "Step ($\xi$): The abstract logical move of ``calculating weekly pages.'' Token Realization ($K$): The specific tokens "3 * 2 = 6". Minimum ($m \in M_q$): A future state where the final answer ``624 pages/year'' is reached. Reasoning Loss ($C(s,q)$): A theoretical measure of how many logical steps remain to reach the answer.} 
We now turn to prove the continuum nature of the reasoning-level state space, thus yielding the rationality of CoT-Space. 

\begin{figure*}[tp]
    \centering
    \includegraphics[width=\linewidth]{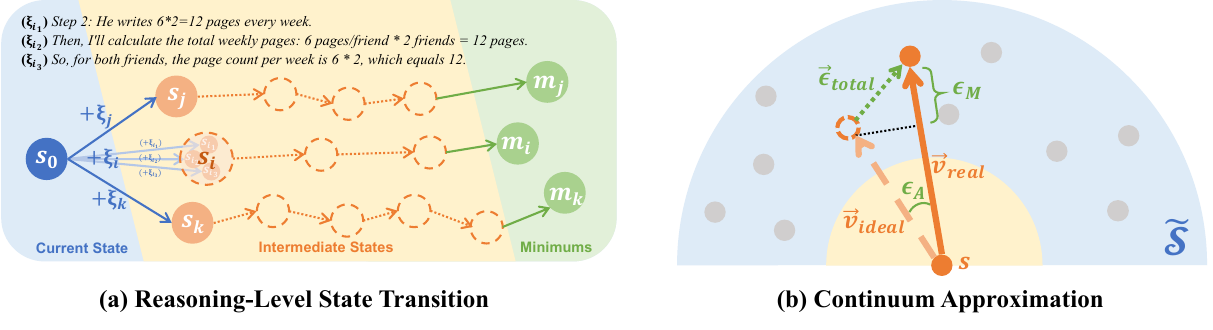}
    %\caption{\textbf{Illustration of CoT-Space.} (a) The reasoning states transit by applying different reasoning steps (such as $\xi_i,\xi_j$ and $\xi_k$) and finally converge to distant minimums ($m_i, m_j$ and $m_k$), while one reasoning step can be realized by multiple distinct token sequences ($\xi_{i_1}, \xi_{i_2}$ and $\xi_{i_3}$). (b) Visualization of continuum approximation. The reasoning state space $\mathcal{S}$ is represented by the dots, while its corresponding continuous manifold $\tilde{\mathcal{S}}$ is represented by the blue region. The ideal and real optimization vectors are represented by orange dashed and solid arrows respectively. }
    \caption{\textbf{Illustration of CoT-Space.} (a) Different reasoning steps ($\xi_i, \xi_j, \xi_k$) lead to different minimums ($m_i, m_j, m_k$), while one step can have multiple token-level realizations ($\xi_{i_1}, \xi_{i_2}, \xi_{i_3}$). (b) The discrete state space $\mathcal{S}$ (dots) is approximated as a continuous manifold $\tilde{\mathcal{S}}$ (blue area), where the real optimization vector ($\vec{v}_\text{real}$) approximates the ideal one ($\vec{v}_\text{ideal}$).} 

    \label{fig:cot-space}
    \vspace{-10pt}
\end{figure*}

\subsection{On the Continuum Nature of the Reasoning-Level State Space}
\label{subsec:on-the-continuum-nature-of-the-reasoning-level-state-space}

%Building upon Assertion~\ref{assertion:huge-and-discontinuous}, which characterizes the token-level state space as both high-dimensional and discrete, we now provide a theoretical basis for approximating this space as a continuum at the reasoning level. The core insight is that while the token-level space consists of discrete points, the \textit{semantic} space it represents becomes increasingly dense as the reasoning process unfolds.

% The essence of reasoning lies in the semantic meaning of conceptual steps, not in the specific tokens used to express them. A single reasoning step, such as ``calculating the total weekly pages'' of the instance in Figure~\ref{fig:internal-slow-thinking}, can be realized by a vast number of distinct token sequences:
% \begin{itemize}
%     \item \textit{``Step 2: He writes 6*2=12 pages every week.''}
%     \item \textit{``Then, I'll calculate the total weekly pages: 6 pages/friend * 2 friends = 12 pages.''}
%     \item \textit{``So, for both friends, the page count per week is 6 $\times$ 2, which equals 12.''}
% \end{itemize}

The essence of reasoning lies in the semantic meaning of conceptual steps, not in the specific tokens used to express them. A single reasoning step, such as ``calculating the total weekly pages'' %of the instance in Figure~\ref{fig:internal-slow-thinking}, 
can be realized by a vast number of distinct token sequences $\xi_{i_1}, \xi_{i_2}$ and $\xi_{i_3}$, as shown in Figure~\ref{fig:cot-space}(a). 

From token-level perspective, different token sequences correspond to distant, discrete states. However, in a higher-dimensional semantic space, they are nearly identical. %, occupying a very small neighborhood. 
The CoT-Space framework posits that as the reasoning length increases, the number of semantically meaningful states grows so rapidly that the space they inhabit can be treated as a continuous manifold. This density arises from the expressive redundancy of language: the number of valid token-level paths to achieve a reasoning goal grows exponentially with the available token budget.

To formalize this, we introduce the following definitions and a key assumption. 

\begin{definition}\label{def:semantic-equivalence-set}
(Semantic equivalence set.) For an abstract reasoning step $\xi_l$, its semantic equivalence set $\mathcal{V}(\xi_l, k)$ is the collection of all token sequences $\tau \in \mathcal{A}^{k}$ with length $k$ that are semantically equivalent to $\xi_l$, defined as: 
$
    \mathcal{V}(\xi_l, k) = \{ \tau \in \mathcal{A}^{k} \ | \ \text{Decode}(\tau) \equiv \xi_l \}, 
$
where $\mathcal{A}^{k}$ is the space of all token sequences with length $k$, and $\text{Decode}(\cdot)$ is an abstract function mapping a token sequence to its semantic meaning. 
\end{definition}

\begin{definition}\label{def:reasoning-state-density}
(Reasoning state density.) Within a fixed semantic volume $\mathcal{V}_{\text{semantic}}$, the set of all valid, reachable reasoning states of reasoning token amount $K$ is denoted by $\mathcal{S}_{\text{reasoning}}^{(K)}$. The state density $\rho(K)$ is defined as the number of states per unit of semantic volume: 
$
    \rho(K) = |\mathcal{S}_{\text{reasoning}}^{(K)}|/|\mathcal{V}_{\text{semantic}}|. 
    % \rho(K) = \frac{|\mathcal{S}_{\text{reasoning}}^{(K)}|}{|\mathcal{V}_{\text{semantic}}|}. 
$
\end{definition}

To substantiate our claim that this density grows exponentially, we propose the following assumption.

\begin{assumption}\label{asm:exponential-expressive-redundancy}
(Exponential expressive redundancy of language.) For any non-trivial reasoning step $\xi_l$, the size of its semantic equivalence set $|\mathcal{V}(\xi_l, k)|$, when realized by $k$ tokens, grows exponentially with $k$. That is, there exists a constant $c > 1$ such that:
$
    |\mathcal{V}(\xi_l, k)| \ge c^k. 
$
\end{assumption}
%This assumption is grounded in the combinatorial nature of language. Allowing even a few extra tokens creates an exponentially larger space of semantically equivalent expressions. 
We provide a proof for this assumption via the Kolmogorov Complexity from both a construction perspective and an information-theoretic perspective in Appendix~\ref{app:proof-of-exponential-expressive-redundancy}. 
This assumption acts as a foundational theoretical posit for our framework, grounded in the combinatorial nature of language where adding tokens exponentially increases the space of semantically equivalent expressions. %While treated here as a structural condition, we acknowledge that this hypothesis invites empirical validation, for instance, by measuring nearest-neighbor distances of paraphrased reasoning steps in a semantic embedding space. We discuss this potential validation methodology in detail in Appendix~\ref{app:empirical-validation-plan}. 

Based on these foundations, we can now formally state the lemma that justifies the continuum approximation of the reasoning-level state space.

% \begin{theorem}\label{thm:continuum-convergence-of-reasoning-level-state-space}
% (Continuum convergence of reasoning-level state space.)
% Under Assumption~\ref{asm:exponential-expressive-redundancy}, the reasoning-level state density, denoted as $\rho(K)$, increases exponentially with the total reasoning token amount $K$. This exponential growth implies that the expected semantic distance between any two adjacent valid reasoning states decreases dramatically as $K$ increases, 
% and approaches zero as $K \to \infty$. Therefore, in the macroscopic limit, the discrete reasoning-level state space converges to a continuous semantic manifold. 
% \end{theorem}

\begin{lemma}\label{lm:continuum-convergence-of-reasoning-level-state-space}
(Continuum convergence of reasoning-level state space.)
Under Assumption~\ref{asm:exponential-expressive-redundancy}, in a $D$-dimensional semantic space, 
the reasoning-level state density, denoted as $\rho(K)$, increases exponentially with the reasoning token amount $K$, formally $\rho(K) = \Theta(c^K)$. 
%Consider two adjacent points $s_i, s_j$ that are nearest neighbor of each other, 
Let $s_i \bowtie s_j$ denotes that $s_i, s_j$ are nearest neighbor states of each other, 
the expected nearest neighbor distance of the semantic manifold is inversely related to the reasoning state density, formally $\mathbb{E}_{s_i \bowtie s_j}[\operatorname{dist}(s_i, s_j)] \propto \rho(K)^{-1/D}$. 
%This exponential growth implies that the expected semantic distance between any two adjacent valid reasoning states decreases dramatically as $K$ increases, 
%and approaches zero as $K \to \infty$. Therefore, in the macroscopic limit, the discrete reasoning-level state space converges to a continuous semantic manifold. 
\end{lemma}

The proof is provided in Appendix~\ref{app:proof-of-continuum-convergence-of-reasoning-level-state-space}. 
%Theorem~\ref{thm:continuum-convergence-of-reasoning-level-state-space} provides the formal justification for the central premise of our work. This result licenses our departure from the discrete, token-level MDP formulation and enables the adoption of more powerful, analytically tractable tools from continuous optimization theory. 
In Lemma~\ref{lm:continuum-convergence-of-reasoning-level-state-space}, we established that the reasoning-level state space converges to a continuous semantic manifold in a macroscopic limit. %While this justifies the conceptual shift, the application of differential tools requires a more rigorous analysis of the approximation error at a local level. 
Here, we further quantify the total ``gap'' between an ideal continuous update and the best possible discrete step, demonstrating that this gap vanishes as the total reasoning token amount $K$ increases. 

A visualization for this analysis is presented in Figure~\ref{fig:cot-space}(b). 
Consider the original discrete reasoning state space $\mathcal{S}$ and its corresponding continuous manifold $\tilde{\mathcal{S}}$, where $\mathcal{S} \subset \tilde{\mathcal{S}}$. 
For a state transition starting from $s$, 
let $\vec{v}_{\text {ideal }}$ be the ideal optimization vector with domain on $\tilde{\mathcal{S}}$, and let $\vec{v}_{\text {real }}$ be the best achievable discrete optimization vector with domain on $\mathcal{S}$. 
Since $\vec{v}_{\text {real }}$ serves as an approximation for $\vec{v}_{\text {ideal }}$, 
the total vector error is $\vec{\epsilon}_{\text {total }}(s)=\vec{v}_{\text {real }}-\vec{v}_{\text {ideal }}$. 
We can decompose the analysis of this error into two components: the error in direction (angle) and the error in size (magnitude).

%let $\vec{v}_{\text {ideal }}=-\eta \nabla C(s)$ be the ideal step vector for some step size $\eta > 0$. Let $\vec{v}_{\text {real }}=s_j^*-s$ be the best achievable discrete step vector, where $s_j^*$ is the valid subsequent state that best approximates the ideal transition. The total vector error is $\vec{\epsilon}_{\text {total }}(s)=\vec{v}_{\text {real }}-\vec{v}_{\text {ideal }}$. We can decompose the analysis of this error into two components: the error in direction (angle) and the error in size (magnitude).

\begin{definition}\label{def:continuum-errors}
(Continuum errors.) 
The total error for continuum approximation is characterized by two components: \textbf{(1) The angular error}, $\epsilon_A(s)$, defined as the angle between $\vec{v}_{\text {ideal }}$ and $\vec{v}_{\text {real }}$. \textbf{(2) The magnitude error}, $\epsilon_M(s)$, defined as difference in norms: $\epsilon_M(s) = \left|\left\|\vec{v}_{\text {real }}\right\|-\left\|\vec{v}_{\text {ideal }}\right\| \right|$.
% \begin{itemize}
% \item The \textbf{angular error}, $\epsilon_A(s)$, defined as the angle between $\vec{v}_{\text {ideal }}$ and $\vec{v}_{\text {real }}$.
% \item The \textbf{magnitude error}, $\epsilon_M(s)$, defined as difference in norms: $\epsilon_M(s) = \left|\left\|\vec{v}_{\text {real }}\right\|-\left\|\vec{v}_{\text {ideal }}\right\| \right|$.
% \end{itemize}
\end{definition}

We will further show the convergence relationship between the continuum errors and the density, thereby revealing the rationality of the reasoning-level state space continuous approximation. 

\begin{theorem}\label{thm:convergence-of-continuum-errors}
(Convergence of the continuum error). The expected angular error $\mathbb{E}\left[\epsilon_A(s)\right]$ and expected magnitude error $\mathbb{E}\left[\epsilon_M(s)\right]$ are both upper-bounded by a function that is inversely related to the reasoning state density $\rho(K)$, formally 
$
\mathbb{E}[\epsilon_A(s)] \le \mathcal{O}\left(\rho(K)^{-\frac{1}{D}}\right)$ and $\mathbb{E}[\epsilon_M(s)] \le \mathcal{O}\left(\rho(K)^{-\frac{1}{D}}\right). 
$
\end{theorem}

The proof is provided in Appendix~\ref{app:proof-of-convergence-of-continuum-errors}. Theorem~\ref{thm:convergence-of-continuum-errors} demonstrates that the total error in approximating a true gradient vector with the best possible discrete step is bounded and converges to zero as the model's expressive capacity ($K$) increases. By ensuring that for a sufficiently large $K$, a discrete step vector $\vec{v}_{\text {real }}$ that is an arbitrarily accurate approximation of the ideal vector $\vec{v}_{\text {ideal }}$ can always be found. 
With the continuum nature of the reasoning space established, we bridge the gap between discrete language generation and classical results. %This continuum approximation is not merely a conceptual shift, it is the critical step that licenses the use of powerful analytical tools. %from continuous mathematics. 
Specifically, it allows us to conceptualize a smooth reasoning loss landscape and thus define a corresponding gradient, enabling us to model the generation of a new reasoning step as a descent-like update. 
%With the continuum nature of the reasoning space established, we are now positioned to formally introduce our framework. In the subsequent section, we will leverage this continuum perspective to define the CoT-Space and reframe the LLM reasoning process as an optimization problem within this newly characterized landscape.

% With the continuum nature of the reasoning space established, we can now model the LLM reasoning process as an optimization problem rather than a token-level MDP. 
% This result licenses our subsequent use of continuous analytical tools. 

Building upon this, we can frame the entire reasoning process as a trajectory in a continuous state space, and link many LLM phenomena, such as hallucination, prompt sensitivity, emergent abilities, and effectiveness of external slow-thinking strategies, with classical theoretical results. A brief discussion is provided in Appendix~\ref{app:broader-exp-of-cot-space}. In the next section, we take the ``overthinking'' phenomenon as a detailed example to reveal the potential of CoT-Space for revitalizing the classical theoretical results. 

%Building upon this, we then frame the entire reasoning process as a trajectory in a continuous state space, and taking the ``overthinking'' phenomenon as an example, thus revealing the potential of CoT-Space to revitalize the classical learning theory results. 

%making it amenable to analysis with tools like Stochastic Differential Equations (SDEs). Without this continuous abstraction, these concepts would be ill-defined in the discrete, combinatorial token space, rendering the subsequent analysis in Section 3 untenable.

%Taking the ``overthinking'' phenomenon as an example, we thus show the potential of CoT-Space to revitalize the classical learning theory results. 

%This formulation simplifies the analytical foundation by reframing the process as an update to the reasoning state, guiding it toward one of the solution minima. 
\section{Analytical Applications within CoT-Space}
\label{sec:analysis}
% This section demonstrates the analytical utility of the CoT-Space framework by applying it to a critical issue in LLM reasoning: the existence of an optimal thought process length and the related problem of ``overthinking''. Moving beyond traditional analyses, we utilize the continuous perspective of CoT-Space to provide a theoretical foundation for this empirical phenomenon. Our analysis proceeds in two parts: we first examine the reasoning process through the lens of optimization and noise (Subsection~\ref{subsec:perspective-of-noise}), and second, we assess its generalization capabilities from the perspective of risk (Subsection~\ref{subsec:perspective-of-risk}). Finally, we briefly conclude our theoretical insights in Subsection~\ref{subsec:takeaways}. This application of CoT-Space showcases its potential to ground empirical findings in a robust theoretical model. 

This section utilizes the CoT-Space framework and combines classic learning theories to analyze the ``overthinking'' phenomenon, thus yielding the potential of CoT-Space. We provide a theoretical foundation by examining the issue through two lenses: optimization noise (Subsection~\ref{subsec:perspective-of-noise}) and reasoning risk (Subsection~\ref{subsec:perspective-of-risk}), followed by a summary of our theoretical insights (Subsection~\ref{subsec:takeaways}). 

%This section applies the CoT-Space framework to analyze the ``overthinking'' phenomenon by revitalizing classical learning theories. We provide a theoretical foundation by examining it from the perspectives of optimization noise (Subsection~\ref{subsec:perspective-of-noise}) and generalization risk (Subsection~\ref{subsec:perspective-of-risk}), culminating in a summary of our findings (Subsection~\ref{subsec:takeaways}). 

\subsection{Reasoning as Optimization: A Perspective of Noise}
\label{subsec:perspective-of-noise}
\begin{figure*}[tp]
    \centering
    \includegraphics[width=\linewidth]{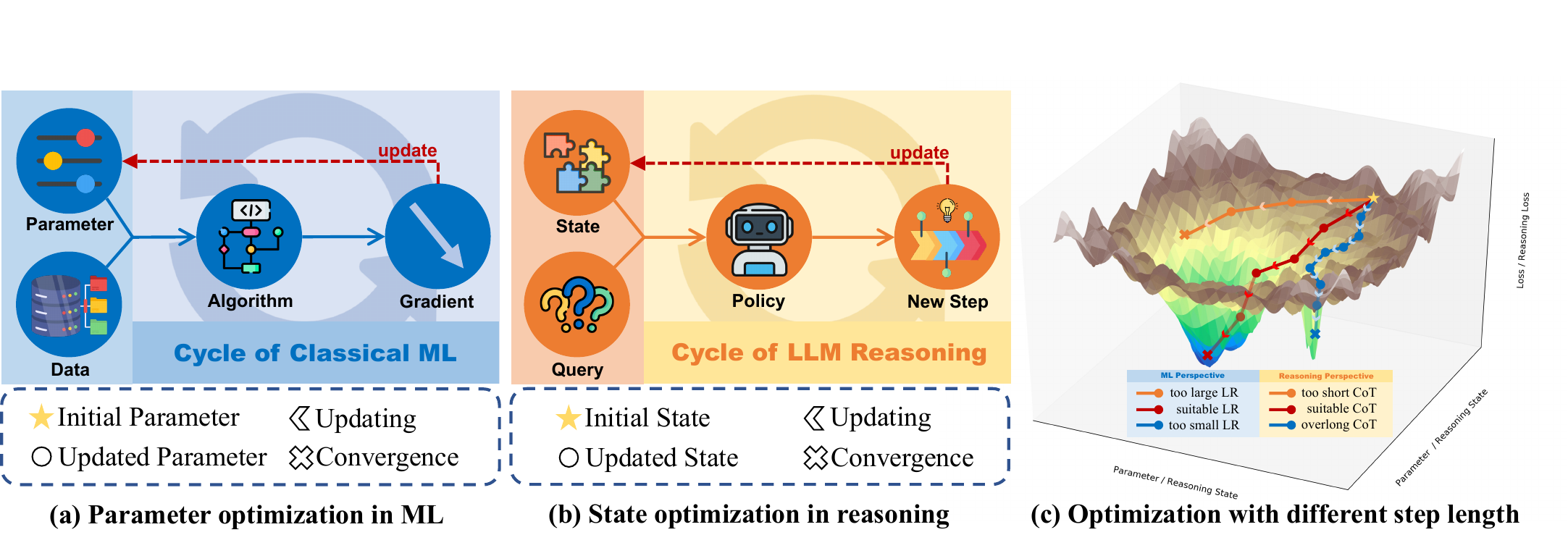}
    \caption{\textbf{Modeling LLM reasoning as an optimization process.} (a) Traditional ML performs parameter optimization by updating model weights. (b) Our framework models reasoning as state optimization, where new reasoning steps iteratively update the current state. (c) This panel visualizes this process on a reasoning loss landscape, where the CoT length is analogous to a learning rate; an optimal value is crucial for converging to a high-quality, generalizable solution.} 
    %\caption{\textbf{An illustration of LLM reasoning modeled as an optimization process.} (a) In traditional ML, the goal is parameter optimization, where an algorithm iteratively calculates a gradient from data to update model parameters and converge to a minimum in the loss landscape. (b) In the CoT-Space framework, LLM reasoning is framed as state optimization, where the policy takes the current state and instruction to generate a new reasoning step, thereby updating the state to move closer to a solution minimum. (c) This panel visualizes the optimization process on a reasoning loss landscape, analogizing the CoT length to the learning rate in ML. A suitable CoT length, much like an optimal learning rate, navigates the space effectively to converge to a high-quality, generalizable solution. }
    \label{fig:reasoning-as-optimization}
    \vspace{-10pt}
\end{figure*}

Leveraging the continuous nature of CoT-Space, we frame LLM reasoning as an optimization process analogous to ML, as shown in Figure~\ref{fig:reasoning-as-optimization}. In this analogy, an LLM policy iteratively refines its reasoning state towards a solution, much like an optimizer updates parameters. 
This perspective allows us to apply further analysis via noise. %classical theories like sharpness-aware minimization, which suggests that an optimal amount of noise promotes convergence to better solutions. 
%The CoT length functions as a learning rate that controls optimization noise: if too short, the process fails to converge; if too long, it leads to ``overthinking'', akin to finding a sharp, poor-quality minimum. %A suitable CoT length is therefore crucial for guiding the reasoning process to a high-quality, generalizable solution. 

%To characterize the random exploration during inference time, we model the optimization of LLM reasoning as a noisy reasoning loss descent process and perform a theoretical analysis of the relationship between the noise scale and CoT length. We model the reasoning process as a Stochastic Differential Equation (SDE) on the continuous manifold (details in Appendix~\ref{app:proof-of-g-inverse-prop-L}). In this formulation, $\sigma$ represents the noise scale of the optimization trajectory. We first give the following theorem as an intuitive result. 
%in the reasoning process. 

To characterize the random exploration during inference time, we model the optimization of LLM reasoning as a noisy reasoning loss descent process. Specifically, we treat the step-by-step generation of a CoT with length $L$ as a discretization of a continuous stochastic process. We formulate this mathematically as a Stochastic Differential Equation (SDE) on the continuous manifold (details in Appendix~\ref{app:proof-of-g-inverse-prop-L}), where $\sigma$ represents the noise scale of the optimization trajectory. 
To theoretically derive the relationship between $\sigma$ and $L$, we bridge the discrete and continuous perspectives by equating the noise variance of a single reasoning step, and give the following theorem as an intuitive result. %This alignment reveals that to maintain consistent exploration dynamics, the intrinsic noise scale must adapt to the granularity of the reasoning steps. 

\begin{theorem}\label{thm:g-inverse-prop-L}
    (Relationship between noise scale and CoT length.) In a given optimization space, the noise scale $\sigma$ is inversely proportional to the reasoning length $L$, i.e. $\sigma \propto \frac{1}{L}$. 
\end{theorem}

The proof is provided in Appendix~\ref{app:proof-of-g-inverse-prop-L}. 
Theorem~\ref{thm:g-inverse-prop-L} clarifies the relationship between the noise $\sigma$ and the CoT length $L$ in the reasoning process. 
According to the sharpness-aware minimization principles~\citep{keskar2016large}, given a reasoning task, there must exist an optimal noise scale $\sigma_\text{opt}$, which corresponds to an optimal CoT length $L_\text{opt}\propto \frac{1}{\sigma_\text{opt}}$. In essence, Theorem~\ref{thm:g-inverse-prop-L} provides a fundamental basis for why an optimal CoT length is crucial for effective reasoning, highlighting its role in introducing the right amount of noise to achieve a flat, high-quality, and generalizable solution. 

% \textcolor{blue}{We note that while real-world optimization dynamics involve complex, state-dependent noise, modeling $\eta(t)$ as white noise serves as a necessary first-order approximation. This simplification allows us to isolate the fundamental inverse relationship between CoT length and the effective noise scale ($\sigma \propto 1/L$) without the obfuscation of higher-order interference. This relationship holds conceptually: longer reasoning sequences naturally integrate more variance, necessitating a lower noise scale per step to maintain stability.} 

%\subsection{Reasoning and Generalization: A Perspective of Risk}

\subsection{Reasoning and Generalization: From Noise to Risk}
\label{subsec:perspective-of-risk}
The previous section fundamentally considered reasoning from the perspective of noise. 
We now provide a deeper analysis to link this noise perspective with the risk in CoT reasoning. By modeling the reasoning process as a stochastic dynamical system, we derive an information-theoretic generalization bound that explicitly characterizes the trade-off between optimization drift and information complexity. 
%``underthinking'' and ``overthinking''. 

\subsubsection{Stochastic Dynamics of Reasoning}
We formalize the generation of a CoT sequence as a discrete-time approximation of a continuous stochastic process on a semantic manifold $\mathcal{M} \subseteq \mathbb{R}^d$. Let $s_t \in \mathbb{R}^d$ denote the reasoning state at step $t$, and $q \sim \mathcal{D}$ be the input query drawn from a data distribution $\mathcal{D}$. The Reasoning Loss $C(s, q): \mathbb{R}^d \times \mathcal{Q} \to \mathbb{R}$ is the semantic distance between the current state $s$ and the ground-truth solution for query $q$. 

To visually illustrate our conclusions, we performed a more refined discretization model of SDE for equation~(\ref{eq:reasoning-SDE}) as follows. It's worth noting that we can easily obtain similar conclusions under continuous settings. 
$$
    s_t = s_{t-1} - \eta \nabla C(s_{t-1}, q) + \sigma \zeta_t, \quad \zeta_t \sim \mathcal{N}(0, I_d),
$$
where $\eta > 0$ is the step size, and $\sigma \ge 0$ represents the Reasoning Noise Scale. The drift term $-\nabla C(s_{t-1}, q)$ drives the reasoning process towards semantic correctness, while the diffusion term $\sigma \zeta_t$ introduces necessary exploration. 

To rigorously analyze the impact of reasoning noise, we first formally distinguish between the deterministic (greedy) reasoning trajectory and the stochastic (exploratory) reasoning trajectory derived from our SDE framework, following the settings of~\citet{wang2022on}. 

\begin{definition}\label{def:noise-free-trajectory}
    (Noise-free trajectory.) Let $\mathcal{T}^* = \{s_0^*, s_1^*, \dots, s_T^*\}$ denote the deterministic reasoning trajectory generated when the noise scale is set to zero ($\sigma = 0$). This trajectory corresponds to a greedy decoding process or a standard gradient descent path in the semantic space, governed by the recurrence:
    $$
        s_t^* = s_{t-1}^* - \eta \nabla C(s_{t-1}^*, q), \quad \text{with } s_0^* = s_{initial}.
    $$
\end{definition}
We refer to $s_T^*$ as the Deterministic Baseline Solution, which serves as the reference point for our stability analysis. 

\begin{definition}\label{def:noisy-trajectory}
    (Noisy trajectory.) Let $\mathcal{T} = \{s_0, s_1, \dots, s_T\}$ denote the stochastic reasoning trajectory generated under a non-zero noise scale $\sigma > 0$. This trajectory evolves according to the full SDE discretization:
    $$
        s_t = s_{t-1} - \eta \nabla C(s_{t-1}, q) + \sigma \zeta_t, \quad \text{with } s_0 = s_{initial},
    $$
    where $\zeta_t \sim \mathcal{N}(0, I_d)$ represents the independent Gaussian noise injected at each step. %(analogous to sampling randomness). 
\end{definition}
Definitions~\ref{def:noise-free-trajectory} and \ref{def:noisy-trajectory} conceptualize the role of noise in the reasoning process. We subsequently analyze the reasoning risk in the following subsections. 

\subsubsection{Reasoning Risk Analysis via Noise}
A core challenge in reasoning is avoiding overfitting to the superficial patterns of a specific prompt (i.e., prompt sensitivity), which leads to poor generalization on unseen queries. To quantify this, we introduce the concept of Thought Dispersion, inspired by the gradient dispersion in SGD analysis. 

\begin{definition}
    (Thought dispersion.) The Thought Dispersion $\mathbb{V}_t(s)$ at state $s$ quantifies the variance of the reasoning direction induced by different queries relative to the expected semantic drift: 
    $$
        \mathbb{V}_t(s) \triangleq \mathbb{E}_{q \sim \mathcal{D}} \left[ \left\| \nabla C(s, q) - \mathbb{E}_{q' \sim \mathcal{D}}[\nabla C(s, q')] \right\|_2^2 \right]. 
    $$
\end{definition}
Intuitively, a high $\mathbb{V}_t(s)$ indicates that the reasoning step is highly sensitive to the specific instance $q$, implying a higher risk of memorizing instance-specific shortcuts rather than learning robust reasoning patterns. 
We now state our main theorem, which bounds the expected population risk $R_{pop} \triangleq \mathbb{E}_{q, \zeta}[C(s_T, q)]$ in terms of the noise scale $\sigma$. 

\begin{theorem}\label{thm:noise-generalization-trade-off}
    (Noise-Generalization trade-off.) Assume the reasoning loss function $C(\cdot, q)$ is $\Gamma$-subguassian and $\beta$-smooth. For a reasoning process of length $L$, the expected population risk is upper-bounded by: 
    $$
        \tiny \mathbb{E}(R_{pop}) \le \underbrace{R_{emp}^* + \frac{\beta L \sigma^2 d}{2}}_{\text{(I) Optimization Drift}} + \underbrace{\sqrt{\frac{2\Gamma^2}{n} \sum_{t=1}^{L} \frac{d}{2} \log \left( 1 + \frac{\eta^2 \mathbb{E}[\mathbb{V}_t(s_{t-1})]}{d \sigma^2} \right)}}_{\text{(II) Information Complexity}},
    $$
    where $n$ is the number of training examples, $R_{emp}^*$ is the empirical risk of the deterministic (greedy) path, and the expectation in Term (II) is taken over the reasoning trajectory.
\end{theorem}

The proof is provided in Appendix~\ref{app:proof-of-noise-gen-trade-off}. Theorem~\ref{thm:noise-generalization-trade-off} reveals that the expected population risk is governed by a fundamental trade-off between two competing components, which we formalize as (I) Optimization Drift and (II) Information Complexity. 
This trade-off is governed by the reasoning noise scale $\sigma$ and the reasoning length $L$. Recall Theorem~\ref{thm:g-inverse-prop-L}, we can roughly derive that the order of Term (I) is $\mathcal{O}(\sigma)$ or $\mathcal{O}(\frac{1}{L})$, and the order of Term (II) is $\mathcal{O}\left(\sqrt{\frac{1}{\sigma} \log \frac{1}{\sigma}}\right)$ or $\mathcal{O}\left(\sqrt{L \log L}\right)$. 

\textbf{Term (I): Optimization Drift} (Underfitting Regime). Analogous to Bias, this term quantifies the deviation from the optimal descent path. As $\sigma \to \infty$, this term dominates as excessive noise disrupts the effective semantic drift. The trajectory is forced to diverge from the correct solution, preventing convergence and resulting in high empirical risk. 

\textbf{Term (II): Information Complexity} (Overfitting Regime). Analogous to Variance, this term bounds the mutual information between the reasoning path and the input. As $\sigma \to 0$ (approximating greedy decoding), this term diverges, meaning the policy encodes excessive instance-specific details. This leads to brittle reasoning chains that are highly sensitive to prompt perturbations, resulting in poor generalization. 

% \textbf{Regime I:} Overfitting (Low $\sigma$). As $\sigma \to 0$ (approximating greedy decoding), Term (II) diverges. The mutual information between the reasoning path and the specific prompt becomes maximized, leading to brittle reasoning chains that fail to generalize to slightly perturbed queries (high prompt sensitivity). 

% \textbf{Regime II:} Underfitting (High $\sigma$). As $\sigma \to \infty$, Term (I) dominates. The excessive noise disrupts the semantic drift, causing the reasoning trajectory to diverge from the correct solution, resulting in high empirical risk. 

Consequently, there exists a strictly positive optimal noise scale $\sigma^* > 0$ that minimizes the total bound. This provides a theoretical justification for the empirical success of sampling-based strategies (e.g., Self-Consistency) over greedy decoding, suggesting that controlled stochasticity is a prerequisite for robust reasoning. We also provide an alternative risk analysis in Appendix~\ref{app:alternative-risk-analysis} for similar conclusions. 

\subsection{Takeaways}
\label{subsec:takeaways}
% We have analyzed with the CoT-Space from two perspectives: that of noise (Subsection \ref{subsec:perspective-of-noise}) and risk (Subsection \ref{subsec:perspective-of-risk}). Based on the theoretical insights from Theorems \ref{thm:g-inverse-prop-L}, \ref{thm:mi-reasoning-gen-err-upbd}, and \ref{thm:empirical-risk-lwbd}, we can draw the following key conclusions about the factors that govern the convergence of the optimal CoT length $L_\text{opt}$ in LLM reasoning via RL.  
Our previous analysis from the perspectives of noise (Subsection~\ref{subsec:perspective-of-noise}) and risk (Subsection~\ref{subsec:perspective-of-risk}) 
yields the following key conclusions about the factors governing the convergence of the optimal CoT length, $L_\text{opt}$, in LLM reasoning via RL. 

\textbf{Remark 1. }
The intrinsic difficulty of a task dictates $L_\text{opt}$ via stability requirements. Mathematically, complex tasks are characterized by a larger smoothness parameter $\beta$ (indicating a rugged loss landscape). Per Term (I) in Theorem~\ref{thm:noise-generalization-trade-off}, a large $\beta$ amplifies the stability cost, necessitating lower optimization noise ($\sigma$) to prevent the empirical risk from exploding. Consequently, the policy must suppress noise by extending the reasoning process, converging to a longer $L_\text{opt}$ to ensure precise convergence. 

%The intrinsic difficulty of a task dictates $L_\text{opt}$. Per Theorem~\ref{thm:empirical-risk-lwbd}, more complex tasks require a greater reasoning depth ($L^*$) to avoid underfitting. Consequently, an effective policy must learn a sufficiently long $L_\text{opt}$ commensurate with the task's complexity to minimize empirical loss. 

%The intrinsic difficulty of a reasoning task is a primary determinant of the optimal CoT length. According to Theorem~\ref{thm:empirical-risk-lwbd}, a more challenging task necessitates a greater $L^{*}$ to achieve a low empirical loss and avoid underfitting. Consequently, the optimal CoT length, $L_\text{opt}$, must be sufficiently large to accommodate the task's complexity, otherwise the policy's failure rate will remain high. This insight formalizes the intuitive idea that complex problems require more thinking steps to be solved correctly, and any effective policy must learn to generate a CoT of a length commensurate with the problem's difficulty. 

\textbf{Remark 2. }
Model capacity is inversely related to the optimal CoT length due to noise regularization. According to Term (II) in Theorem~\ref{thm:noise-generalization-trade-off}, higher-capacity models exhibit greater thought dispersion ($\mathbb{V}_t$), increasing the generalization gap. To mitigate this, they require a higher noise scale $\sigma$ to mask prompt sensitivities. This necessitates a shorter, more concise $L_\text{opt}$ to introduce the required stochasticity. 

%Model capacity is inversely related to the optimal CoT length due to overfitting risk. According to Theorem~\ref{thm:mi-reasoning-gen-err-upbd}, higher-capacity models face a greater generalization error, which scales with CoT length $L$. To mitigate this risk, more powerful models learn to produce more concise reasoning, converging to a shorter $L_\text{opt}$ to regularize their outputs. 

%The capacity of the policy model directly influences the optimal CoT length by affecting the risk of overfitting. A model with higher capacity (e.g., more parameters) can explore a more expansive reasoning state space. This expanded exploration enables it to access states with higher reasoning losses that are inaccessible to weaker models, leading to a larger effective $C_{max}$. The generalization bound in Theorem~\ref{thm:mi-reasoning-gen-err-upbd} is an increasing function of both $C_{max}$ and the CoT length $L$. Consequently, to mitigate the heightened overfitting risk posed by its own high capacity and the resulting larger effective $C_{max}$, a more powerful model must learn to regularize its output. It achieves this by converging to a smaller optimal CoT length ($L_\text{opt}$), thereby maintaining a low generalization error. %This finding provides a theoretical basis for why more expressive models must be appropriately constrained to generalize well. 

\textbf{Remark 3. }
The converged CoT length is independent of the specific RL algorithm used. Our analysis indicates that $L_\text{opt}$ is an intrinsic property determined by task difficulty and model capacity. The algorithm's role is more likely to effectively optimize the policy to find this good solution. %, not to define it. 

%While the choice of RL algorithm is crucial for effective training, our analysis suggests it has a limited direct influence on the final, converged CoT length. The optimal CoT length ($L_\text{opt}$) is primarily a function of the policy's capacity and the inherent complexity of the training data. The RL algorithm's main role is to find a policy that converges to this optimal length. As long as the algorithm is effective in its optimization, the final $L_\text{opt}$ is a property of the learned policy and the data distribution, rather than a byproduct of the specific RL approach used.

\textbf{Remark 4. }
The optimization noise scale is inversely correlated with the converged CoT length. Theorem~\ref{thm:g-inverse-prop-L} shows that higher noise levels incentivize convergence to a shorter $L_\text{opt}$, as noisier environments favor more robust, concise reasoning paths that are less susceptible to disruption. 

%In a given problem space, the scale of noise in the optimization process is inversely correlated with the converged CoT length. According to Theorem \ref{thm:g-inverse-prop-L}, when a policy is exploring the same task space, a higher level of noise will lead to a smaller converged optimal CoT length ($L_\text{opt}$). This implies that a noisier optimization environment tends to favor simpler, shorter reasoning paths, likely because longer CoTs are more susceptible to the disruptive effects of noise. This finding underscores the importance of a stable training environment for achieving complex and in-depth reasoning. 

To validate these theoretical claims, the following section presents a series of experiments designed to systematically test each of these four remarks. 

%To substantiate these theoretical derivations and demonstrate their practical relevance, we will now proceed to a series of empirical investigations in the following section. These experiments are designed to systematically test the four central remarks, providing robust empirical evidence for our proposed framework. 

\section{Experiments}
\label{sec:experiments}
\begin{figure*}[tp]
    \centering
    \includegraphics[width=\linewidth]{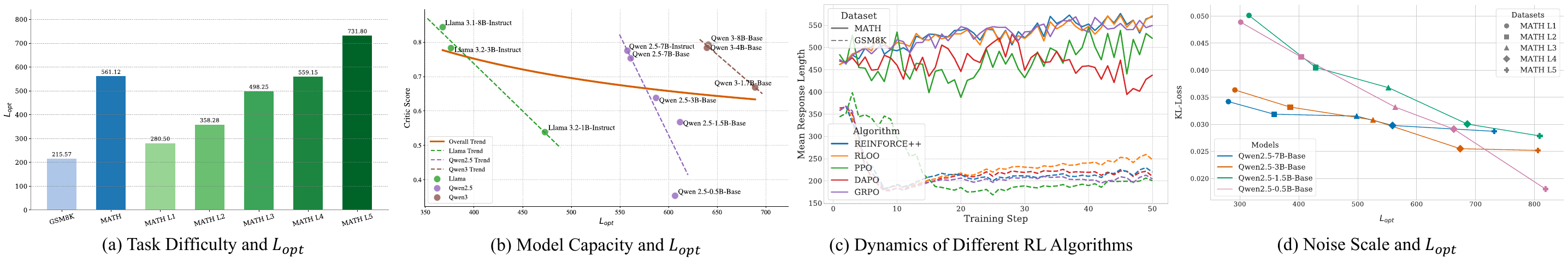}
    \caption{\textbf{Empirical validation of factors governing optimal CoT length ($L_\text{opt}$).} (a) \textbf{Task difficulty:} $L_\text{opt}$ increases with task complexity. (b) \textbf{Model capacity:} More powerful models converge to a shorter $L_\text{opt}$ to mitigate overfitting. (c) \textbf{RL algorithm:} The converged $L_\text{opt}$ is largely algorithm-agnostic. (d) \textbf{Optimization noise:} $L_\text{opt}$ is inversely correlated with noise (proxied by KL-Loss).} 
    %\caption{\textbf{Empirical Validation of the Factors Governing Optimal CoT Length ($L_\text{opt}$).} (a) The relationship between task difficulty and $L_\text{opt}$, demonstrating that more complex tasks require longer reasoning process. (b) The inverse relationship between model capacity and $L_\text{opt}$, where more powerful models converge to shorter CoTs to mitigate overfitting. (c) The convergence of mean response length under different RL algorithms, showing that the final $L_\text{opt}$ is largely algorithm-agnostic. (d) The inverse correlation between optimization noise scale (proxied by KL-Loss) and $L_\text{opt}$, indicates that noisier training environments favor more concise reasoning.}
    \label{fig:experiment}
    %\vspace{-15pt}
\end{figure*}

This section presents a series of experiments designed to empirically validate the theoretical insights established in Section~\ref{sec:analysis}. We systematically investigate the relationship between the optimal CoT length ($L_\text{opt}$) and four key factors. The detailed experimental setup is listed in Appendix~\ref{app:experimental-settings}. %including task difficulty, model capacity, the choice of RL algorithm, and optimization noise scale. %in Subsections~\ref{subsec:task-difficulty-and-l} through~\ref{subsec:noise-and-l}, respectively. 

\textbf{Task Difficulty and $L_\text{opt}$. }
To validate Remark 1, we test on tasks of increasing difficulty (GSM8K and five levels of MATH) with Qwen2.5-7B-Base. As shown in Figure~\ref{fig:experiment}(a), $L_\text{opt}$ demonstrates a clear monotonic increase with task complexity. This aligns with Theorem~\ref{thm:noise-generalization-trade-off}, as more challenging tasks require a greater reasoning depth. To avoid the high empirical loss from underfitting, the policy learns to generate a correspondingly longer CoT. Additional validations are in Appendix~\ref{app:more-experiments}.

% \subsection{Model Capacity and Converged Length}
% \label{subsec:model-capacity-and-l}
% This subsection examines the influence of model capacity on the optimal CoT length, thereby verifying Remark 2. We train the models by GRPO, then use the parameter count of different models as a proxy for their capacity and the final critic score as a measure of their reasoning performance. Figure~\ref{fig:experiment}(b) plots the converged CoT length ($L_\text{opt}$) against this critic score for various models trained on MATH. The results reveal a distinct negative correlation: within each model family (e.g., Llama 3, Qwen 2.5, and Qwen 3), larger and higher-performing models consistently converge to a shorter $L_\text{opt}$. This empirical finding aligns precisely with our theoretical analysis. As stated in Remark 2 and derived from the generalization bound in Theorem~\ref{thm:mi-reasoning-gen-err-upbd}, models with higher capacity are inherently more susceptible to overfitting. The risk of overfitting is amplified by longer and more complex CoTs. Therefore, to minimize the total reasoning error, a more powerful model learns to regularize its reasoning process by adopting a more concise, shorter CoT, effectively balancing the trade-off between empirical performance and generalization. 
\textbf{Model Capacity and $L_\text{opt}$. }
To verify Remark 2, we examine models with varying capacities. Figure~\ref{fig:experiment}(b) reveals a distinct negative correlation: within each model family (e.g., Llama 3, Qwen 2.5, and Qwen 3), larger and higher-performing models consistently converge to a shorter $L_\text{opt}$. This finding supports our analysis based on Theorem~\ref{thm:noise-generalization-trade-off}. Larger models are more prone to overfitting, a risk exacerbated by longer CoTs. Consequently, they learn to produce more concise reasoning to minimize total error by balancing empirical performance with generalization.

% \subsection{RL Algorithm and Converged Length}
% \label{subsec:algs-and-l}
% In this experiment, we assess the impact of different RL algorithms on the final converged CoT length to test the validity of Remark 3. We track the mean response length over the course of training for five different RL algorithms: REINFORCE++, RLOO, PPO, DAPO, and GRPO. The resulting dynamics are displayed in Figure~\ref{fig:experiment}(c). While the convergence trajectories and learning speeds differ markedly among the algorithms—for example, GRPO gradually increases the response length, whereas RLOO maintains a high length from the outset—all tested algorithms ultimately steer the policy toward a similar range of converged mean response lengths for a given task and model. This outcome supports our insight in Remark 3 that the choice of RL algorithm is primarily a facilitator of the optimization process rather than a determinant of the optimal solution's properties. The final converged length, $L_\text{opt}$, appears to be an intrinsic characteristic dictated by the interplay between model capacity and task difficulty, not an artifact of the specific optimization algorithm used to find it. 
\textbf{RL Algorithm and $L_\text{opt}$. }
To test Remark 3, we compare five different RL algorithms. As shown in Figure~\ref{fig:experiment}(c), despite different learning dynamics, all algorithms guide the policy to a similar final $L_\text{opt}$ for a given task and model, even though the initial dynamics differ. This supports our claim that the RL algorithm primarily facilitates optimization, while the final $L_\text{opt}$ is an intrinsic property determined by task difficulty and model capacity. %not the specific algorithm used. 

% \subsection{Noise Scale and Converged Length}
% \label{subsec:noise-and-l}
% Finally, we validate Remark 4 by analyzing the relationship between the converged CoT length and the noise scale of the optimization process. For this analysis, we employ the Kullback-Leibler (KL) divergence loss, or ``KL-Loss'', as a quantifiable proxy for optimization noise. A higher KL-Loss indicates a greater deviation from the reference policy, representing a less stable and thus ``noisier'' update. Figure~\ref{fig:experiment}(d) illustrates the relationship between the final KL-Loss and the converged length ($L_\text{opt}$) across various models and task difficulties. The data shows a clear inverse correlation: for a given task, models that exhibit a higher KL-Loss consistently converge to a shorter $L_\text{opt}$. This empirical result substantiates the relationship derived in Theorem~\ref{thm:g-inverse-prop-L}. A noisier optimization environment penalizes longer, more intricate reasoning paths, as they are more susceptible to disruption from unstable policy updates. Consequently, the RL process naturally favors shorter and more robust CoTs to ensure stable convergence, confirming that a higher noise scale leads to a smaller optimal CoT length. 
\textbf{Noise Scale and $L_\text{opt}$. }
Finally, to validate Remark 4, we use the KL-Loss as a proxy for optimization noise. 
%\textcolor{blue}{In RL algorithms like PPO, the KL divergence term penalizes deviations from the reference policy, effectively regulating the exploration range. A higher KL-Loss indicates a larger deviation and higher variance in the policy's update trajectory, analogous to a higher noise scale in our theoretical model.} 
Figure~\ref{fig:experiment}(d) shows a clear inverse correlation between the final KL-Loss and $L_\text{opt}$. This result substantiates Theorem~\ref{thm:g-inverse-prop-L}, confirming that noisier optimization environments favor shorter CoTs to ensure robust convergence, as longer reasoning paths are more susceptible to disruption. 
\section{Related Work}
\label{sec:related-work}
%\subsection{Reasoning with LLMs}
% Recent advancements in Large Language Models (LLMs) have significantly enhanced their capacity for complex reasoning. Building upon the foundational concept of LLMs as few-shot learners~\citep{brown2020language}, a pivotal development has been the use of prompting techniques to elicit explicit reasoning processes. Chain-of-Thought (CoT) prompting~\citep{wei2022chain}, for instance, guides models to generate intermediate steps, a method whose robustness was later improved by Self-Consistency through the aggregation of multiple reasoning paths~\citep{wang2022self}. Theoretical understanding has advanced in parallel, with information-theoretic frameworks offering methods to quantify and mitigate reasoning errors, thus bridging empirical and theoretical work~\citep{ton2024understanding}. Concurrently, explorations into the representation of reasoning have yielded novel formats, from structured code in Program-of-Thought~\citep{chen2022program} to internal thought tokens generated within the model's latent space~\citep{zelikman2024quiet,chen2025towards}. However, these explorations have been largely empirical, often lacking a strong theoretical foundation. 
\textbf{Reasoning with LLMs. }
Prompting techniques have advanced LLM reasoning beyond few-shot learning~\citep{brown2020language}. Chain-of-Thought (CoT) prompting elicits intermediate reasoning steps~\citep{wei2022chain}, with methods like Self-Consistency improving robustness~\citep{wang2022self}. Other reasoning formats include Program-of-Thought~\citep{chen2022program} and internal thought tokens~\citep{zelikman2024quiet,chen2025towards}. While theoretical analyses are emerging~\citep{ton2024understanding}, much of these remains empirical and lacks a unifying theoretical foundation.

%\subsection{RL for LLM Reasoning}
% Test-time scaling has emerged as an effective technique for enhancing the reasoning capabilities of Large Language Models (LLMs)~\citep{snell2024scaling}, with its underlying inference scaling laws having been empirically validated~\citep{wu2024inference}. Reinforcement Learning (RL) represents a prominent methodology for implementing these scaling techniques~\citep{chen2025empirical}. The application of RL to LLMs has spurred two parallel lines of inquiry. The first focuses on the development of novel frameworks for post-training, such as VC-PPO~\citep{yuan2025s}, DAPO~\citep{yu2025dapo}, and VAPO~\citep{yue2025vapo}. The second line of inquiry seeks to elucidate the fundamental impact of RL on LLM behavior, where researchers have explored diverse questions. For instance, \citet{setlur2025scaling} demonstrated the conditional advantages of RL over Supervised Fine-Tuning (SFT). \citet{yue2025does} empirically investigated whether RL can expand the knowledge boundaries of LLMs, while \citet{shao2025spurious} examined the influence of reward modeling. Concurrently, \citet{zhao2025echo} posited that RL primarily serves to amplify behaviors acquired during pre-training. Although these studies provide valuable insights from different perspectives, the absence of a common theoretical formulation can lead to divergent conclusions. 

\textbf{RL for LLM Reasoning. }
RL is a key method for test-time scaling, a technique proven to enhance LLM reasoning~\citep{snell2024scaling, wu2024inference, chen2025empirical}. Research in this area follows two main tracks: developing new post-training frameworks like VC-PPO, DAPO, and VAPO~\citep{yuan2025s, yu2025dapo, yue2025vapo}, and understanding the fundamental effects of RL. Studies in the latter track investigate RL's advantages over SFT~\citep{setlur2025scaling}, its impact on knowledge boundaries~\citep{yue2025does}, the role of reward modeling~\citep{shao2025spurious}, and its tendency to amplify pre-trained behaviors~\citep{zhao2025echo}. The absence of a common theoretical framework, however, can lead to divergent findings.

\textbf{Optimization and Generalization. }
A core concept in machine learning optimization is that flatter minima in the loss landscape often lead to better generalization, a principle central to Sharpness-Aware Minimization (SAM)~\citep{keskar2016large}. This effect is closely linked to the noise inherent in the optimization process~\citep{smith2018bayesian}. Generalization itself, which measures performance on unseen data, is often analyzed through the powerful lens of information theory. Concepts like entropy and mutual information are used to establish theoretical bounds on the generalization capacity of deep learning models~\citep{russo2019much,xu2017information} and have recently been applied to analyze synthetic data generation and reasoning errors in LLMs~\citep{gan2024towards,ton2024understanding}.

\section{Conclusion}
\label{sec:conclusion}
% In this paper, we addressed the critical misalignment between token-level RL theories and the reasoning-level nature of CoT. We introduced CoT-Space, a novel theoretical framework that reframes LLM reasoning from a discrete token-prediction task to a continuous optimization problem. This shift in perspective serves as a conceptual bridge, demonstrating that the foundational principles of classical machine learning are not obsolete but can be powerfully revitalized to analyze modern LLM behaviors. By applying classical concepts of optimization and risk analysis within our framework, we demonstrate that a fundamental trade-off governs the reasoning process: shorter CoTs risk underfitting, while longer CoTs risk overfitting. This analysis provides a solid theoretical grounding for the existence of an optimal CoT length, $L_{opt}$, which minimizes this total error.

% Looking forward, the CoT-Space framework opens up several promising avenues for future research. It can guide the design of more principled RL algorithms and new regularization techniques aimed at explicitly controlling CoT complexity. Furthermore, this framework could be extended to other forms of complex reasoning beyond CoT, advancing our theoretical understanding and practical ability to build more capable and reliable AI reasoning systems. 

In this paper, we introduced CoT-Space, a novel theoretical framework that addresses the critical misalignment between token-level reasoning formulation and the reasoning-level nature of CoT by reframing LLM reasoning as a continuous state optimization problem. This shift in perspective serves as a conceptual bridge, demonstrating that the foundational principles of classical machine learning are not obsolete but can be powerfully revitalized to analyze modern LLM behaviors. 
By revitalizing classical concepts of noise and risk analysis, we demonstrate that a fundamental trade-off between underfitting (shorter CoTs) and overfitting (longer CoTs) governs the reasoning process. This analysis provides a solid theoretical grounding for the existence of an optimal CoT length, $L_{opt}$, that minimizes total error. 
%\textcolor{blue}{Looking forward, the CoT-Space framework moves beyond description to offer a prescriptive foundation for future algorithms. By establishing the continuum nature of reasoning, our work suggests two promising directions: (1) Latent Space Optimization, where reasoning is performed directly within the continuous manifold via auto-encoders rather than discrete token generation, and (2) Continuous Loss Design, motivating the development of dense verifiers that output smooth ``distance-to-solution'' signals rather than sparse binary rewards. We elaborate on these implications in Appendix~\ref{app:prescriptive-implications}.} 
Looking forward, the CoT-Space framework offers a foundation for designing more principled algorithms and advancing our ability to build more reliable AI reasoning systems.

\section*{Impact Statement}
This paper presents work whose goal is to advance the field of Machine Learning. There are many potential societal consequences of our work, none of which we feel must be specifically highlighted here.

\bibliography{reference}
\bibliographystyle{icml2026}

%%%%%%%%%%%%%%%%%%%%%%%%%%%%%%%%%%%%%%%%%%%%%%%%%%%%%%%%%%%%%%%%%%%%%%%%%%%%%%%
%%%%%%%%%%%%%%%%%%%%%%%%%%%%%%%%%%%%%%%%%%%%%%%%%%%%%%%%%%%%%%%%%%%%%%%%%%%%%%%
% APPENDIX
%%%%%%%%%%%%%%%%%%%%%%%%%%%%%%%%%%%%%%%%%%%%%%%%%%%%%%%%%%%%%%%%%%%%%%%%%%%%%%%
%%%%%%%%%%%%%%%%%%%%%%%%%%%%%%%%%%%%%%%%%%%%%%%%%%%%%%%%%%%%%%%%%%%%%%%%%%%%%%%
\newpage
\appendix
\onecolumn

\section{Limitations of Token-Level Formulation in Natural Language}
\label{app:limitation-of-natural-language}
Consider a CoT reasoning process comprises multiple conceptual steps $(\xi_1,\xi_2,\dots,\xi_L)$, which naturally reveals two distinct procedural layers: 

\textbf{(1) Reasoning-Level:} The strategic planning of abstract reasoning steps $(\xi_1 \rightarrow \xi_L)$. 

\textbf{(2) Token-Level:} The tactical realization of specific steps via token generation. 

Traditional MDP formulations, however, operate only at the problematic token level, failing to capture the higher-order planning inherent to CoT. 

The token-level state space $\mathcal{S}$ is problematic for two reasons. First, it grows exponentially with sequence length ($\mathcal{S}_{t}=\mathcal{A}^{t}$), making it too vast to explore effectively during training. 
Second, the auto-regressive process imposes a discrete and unidirectional topology, ($\forall \bm{s}_t \in \mathcal{S}_t, \mathcal{T}(\bm{s}_t,\pi) \in \mathcal{S}_{t+1}$)
further complicating analysis. 
This narrow focus creates a theoretical barrier that obscures the problem's true nature. To establish a sound foundation, we must therefore transition from a token-level formulation to a reasoning-level perspective. 
\section{An Alternative Risk Analysis from Information-Theoretic Perspective}
\label{app:alternative-risk-analysis}
Here, we analyze the generalization risk of a policy to provide further theoretical insights, focusing on the upper bounds of reasoning risks. 
In the context of CoT-Space, risks can be defined similarly to classical ML. The total expected error of a policy $\pi$ on the true data distribution $\mathcal{D}$, denoted by $R(\pi) = \mathbb{E}_{q\sim\mathcal{D}}[C(s_q^\pi,q)]$, can be decomposed into two distinct components: $ R(\pi) = \hat{R}_n(\pi) + \left( R(\pi) - \hat{R}_n(\pi) \right)$. 
\begin{equation}
    R(\pi) = \underbrace{\hat{R}_n(\pi)}_{\text{Empirical Loss (Bias)}} + \underbrace{\left( R(\pi) - \hat{R}_n(\pi) \right)}_{\text{Generalization Error (Variance)}}. 
\end{equation}
Here, $\hat{R}_n(\pi) = \frac{1}{n}\sum_{i=1}^n C(s_{q_i}^\pi,q_i)$ is the empirical loss (a.k.a. bias) on the training set. $R(\pi) - \hat{R}_n(\pi)=\mathbb{E}_{q \sim \mathcal{D}}[C(s_{q}^{\pi},q)] - \frac{1}{n}\sum_{i=1}^{n}C(s_{q_i}^{\pi},q_i)$ is the reasoning generalization error (a.k.a. virance). In the following parts, we will analyze how each component behaves as a function of the CoT complexity. 

\subsection{An Upper Bound for Generalization Error: The Overfitting Risk}
%We define the reasoning generalization error as the difference between the true expected loss and the empirical loss. 
Given $n$ queries $\{q_i\}_{i=1}^{n}$ sampled i.i.d. from a distribution $\mathcal{D}$, and a policy $\pi$ learned from $\{q_i\}_{i=1}^{n}$, let's denote $s_{q}^{\pi} = \pi(q)$ as the end reasoning state output by $\pi$ on query $q$. The reasoning generalization error of $\pi$ on $\mathcal{D}$ is defined as: 
$R(\pi) - \hat{R}_n(\pi)=\mathbb{E}_{q \sim \mathcal{D}}[C(s_{q}^{\pi},q)] - \frac{1}{n}\sum_{i=1}^{n}C(s_{q_i}^{\pi},q_i).$ 
This error quantifies a policy's ability to generalize its reasoning to unseen queries, a concept parallel to the generalization error in traditional machine learning theory. 
% \begin{definition}\label{def:reasoning-gen-error}
%     (Reasoning generalization error.) Given $n$ queries $\{q_i\}_{i=1}^{n}$ sampled i.i.d. from a distribution $\mathcal{D}$, and a policy $\pi$ learned from $\{q_i\}_{i=1}^{n}$, let's denote $s_{q}^{\pi} = \pi(q)$ as the end reasoning state output by $\pi$ on query $q$. The reasoning generalization error of $\pi$ on $\mathcal{D}$ is defined as: 
%     $$R(\pi) - \hat{R}_n(\pi)=\mathbb{E}_{q \sim \mathcal{D}}[C(s_{q}^{\pi})] - \frac{1}{n}\sum_{i=1}^{n}C(s_{q_i}^{\pi}).$$ 
% \end{definition}
By applying classical PAC-Bayesian and information-theoretic results of ML, we can derive an upper bound for the reasoning generalization error. 

\begin{theorem}\label{thm:mi-reasoning-gen-err-upbd}
(Information-theoretical generalization upper bound for LLM reasoning.) With CoT-Space defined above, if the policy $\pi$ is trained on an i.i.d. drawn training set $S$ with size $n$, the following upper bound for the reasoning generalization error holds:
$$
    |\mathbb{E}[R(\pi) - \hat{R}_n(\pi)]| \leq \sqrt{\frac{C_{max}^2 \cdot \left(\mathbb{E}_{\pi,S}[K] \right) \cdot \log|\mathcal{A}|}{2n}},
$$
where $C_{max}$ is the maximum possible value of the reasoning loss $C(\cdot)$, $L$ is the number of reasoning steps in a CoT data, $|\xi_i|$ is the number of tokens in the $i$-th reasoning step in a CoT data, $\mathbb{E}[K] \approx \mathbb{E}[L] \cdot \mathbb{E}[|\xi|]$ is the expected total number of tokens in the generated CoT. 
\end{theorem}
A detailed discussion and proof are provided in Appendix~\ref{app:analysis-for-mi-reasoning-error-upbd}. 
Theorem \ref{thm:mi-reasoning-gen-err-upbd} explicitly connects the generalization error to the observable, structural properties of the reasoning process. It suggests that overfitting can occur %in two distinct ways: through excessive reasoning depth (high $\mathbb{E}[L]$) or through excessive step verbosity (high $\mathbb{E}[|\xi|]$), both resulting 
in overlong response length. This provides a direct theoretical basis for regularizing the CoT generation process in RL to improve reasoning generalization. 

It is important to acknowledge that standard generalization bounds typically assume i.i.d. data distributions, whereas RL involves non-stationary, on-policy data generation. However, we apply these classical bounds here as a conceptual proxy to capture the structural relationship between policy complexity (proxied by CoT length $L$) and generalization risk. This theoretical framing allows us to mechanistically interpret ``overthinking'' as a form of overfitting, where the policy memorizes complex reasoning paths rather than learning generalizable logic. Developing tighter bounds specifically for the non-stationary RL setting remains a critical direction for future work. 

\subsection{A Lower Bound for Empirical Loss: The Underfitting Risk}
The empirical loss, $\hat{R}_n(\pi)$, measures a policy's ability to solve the problems it has already encountered. For complex reasoning tasks, a certain minimum number of steps is required for a correct solution. Intuitively,
if $L$ is too small, the policy lacks the necessary ``thinking space'' to solve the queries, even those in the training set. This leads to high empirical loss and underfitting. 
As $L$ increases, the policy has more capacity to find the correct reasoning path, so the empirical loss $\hat{R}_n(\pi)$ generally decreases. 

Building upon this intuition, we can derive a lower bound for the empirical reasoning loss $\hat{R}_n(\pi)$. 

\begin{theorem} \label{thm:empirical-risk-lwbd}
With CoT-Space defined above, if the policy $\pi$ is trained on an i.i.d. drawn training set, the following lower bound for the empirical reasoning risk holds:
$$
    \hat{R}_n(\pi) \ge P_\pi(L < L^*) \cdot C_{fail},
$$
%where $L^{*}$ is the minimum required CoT depth for successful reasoning, $P_\pi(L < L^*) = n_{fail}/n$ is the empirical failure rate of policy $\pi$ on the training set. 
where $L^*$ is the intrinsic required reasoning depth for the query. Note that $L^*$ is a theoretical property of the problem difficulty and is unobservable in practice. This theorem formalizes the intuition that if a policy generates a CoT shallower ($L$) than the problem requires ($L^*$), it is guaranteed to fail (underfit). $C_{fail}$ is the upper bound of the reasoning loss in the failed reasoning results. 
\end{theorem}

A detailed proof is provided in Appendix~\ref{app:proof-for-empirical-risk-lwbd}. 
Theorem~\ref{thm:empirical-risk-lwbd} provides a formal basis for the ``underfitting'' portion of our trade-off curve. It demonstrates that the empirical loss is directly constrained by the probability that the policy generates a CoT that is too short for the problem's intrinsic complexity. A policy that tends to produce short CoTs %(i.e., has a small $\mathbb{E}[L]$) 
will have a high failure rate %$P_\pi(L < L^*)$ 
on %any dataset containing 
non-trivial problems. Consequently, its empirical loss $\hat{R}_n(\pi)$ is guaranteed to be high. This formalizes why a policy must maintain enough CoT length $L$ to achieve a low empirical loss, thus completing our explanation for the tradeoff in our total error analysis and the existence of an optimal $L_\text{opt}$. 

\section{Proofs}
\label{app:proofs}

\subsection{Proof of Assumption~\ref{asm:exponential-expressive-redundancy}}
\label{app:proof-of-exponential-expressive-redundancy}

In this section, we provide a formal justification for Assumption~\ref{asm:exponential-expressive-redundancy}. We demonstrate that for any non-trivial semantic concept, the number of valid token-level realizations grows exponentially with the sequence length. We ground this analysis in Kolmogorov Complexity and the combinatorial properties of natural language. 

Let $\mathcal{A}$ be the discrete vocabulary of tokens, where $|\mathcal{A}| = V$. The space of all token sequences of length $k$ is $\mathcal{A}^k$. Let $\psi$ be a specific abstract reasoning step (a semantic concept). We define the Kolmogorov Complexity for natural language as follows: 

\begin{definition}
    (Kolmogorov complexity for natural language.) We define the Kolmogorov Complexity for natural language, denoted as $K(\psi)$, as the length of the shortest possible description (or program) $p$ that can effectively communicate the semantic meaning of $\psi$ in a universal language~\citep{li2008introduction}: 
    $$
        K(\psi) = \min_{p} \{ |p| : U(p) \equiv \psi \},
    $$
    where $U$ is a universal Turing machine (or in our context, an ideal language decoder), and $\equiv$ denotes semantic equivalence. 
\end{definition}

Let $L_{min}(\psi) = K(\psi)$ be the length of the shortest natural language sequence that maps to $\psi$. Let $p_\psi$ be the shortest realization of the semantic concept $\psi$ where $|p| = L_{min}(\psi)$. We provide the proof for the following lemma which characterizes the exponential growth of semantic equivalence set $\mathcal{V}(p_\psi,k)$: 

\begin{lemma}\label{lm:exponetial-growth-of-semantic-equivalence-set}
    (Exponential growth of semantic equivalence set.) For a semantic concept $\psi$ with minimal description length $L_{min}$, the size of its semantic equivalence set $|\mathcal{V}(p_\psi, k)|$ (the number of distinct token sequences of length $k$ that decode to $\psi$) satisfies:
    $$|\mathcal{V}(\psi, k)| \ge \alpha \cdot c^{k},$$
    for some constants $\alpha > 0$ and $c > 1$, provided $k \ge L_{min}$. 
\end{lemma}

The proof can be derived from two perspectives. We first consider the combinatorial construction via semantic-preserving operations. 
\begin{proof}
(Combinatorial construction via semantic-preserving operations.) 
Let $\tau_0$ be a minimal sequence realizing $\psi$, such that $|\tau_0| = L_{min}$. We can transform $\tau_0$ into longer sequences without altering its core semantics by applying a set of Semantic-Preserving Operations (SPOs), denoted as $\mathcal{O} = \{op_1, op_2, \dots, op_m\}$. 
Common SPOs in natural language include: 
\begin{itemize}
    \item Synonym Replacement: Replacing a token with a synonym (e.g., ``calculate'' $\to$ ``compute''). 
    \item Syntactic Expansion: Adding non-functional modifiers or filler phrases (e.g., ``Therefore'' $\to$ ``It can be clearly seen that''). 
    \item Structural Permutation: Changing active voice to passive voice, or reordering independent clauses. 
\end{itemize}
Let each operation $op_i$ increase the sequence length by an average of $\Delta l$ tokens and have a branching factor of $b$ (i.e., there are $b$ different ways to apply operations to increase length).
To generate a sequence of target length $k$ starting from $\tau_0$, we need to apply approximately $j = (k - L_{min}) / \Delta l$ operations. 

The number of distinct paths to generate a valid sequence of length $k$ is determined by the branching factor $b$ applied $j$ times. Thus, the lower bound on the number of realizations is:
\begin{equation}
    |\mathcal{V}(p_\psi, k)| \ge b^{j} = b^{\frac{k - L_{min}}{\Delta l}}. 
\end{equation}
We can rewrite this term to isolate $k$:
\begin{equation}
    |\mathcal{V}(p_\psi, k)| \ge \left( b^{\frac{1}{\Delta l}} \right)^{k} \cdot \left( b^{-\frac{L_{min}}{\Delta l}} \right). 
\end{equation}
Let $c = b^{1/\Delta l}$ and $\alpha = b^{-L_{min}/\Delta l}$. Since there are multiple valid synonyms and syntactic structures for any non-trivial concept, $b > 1$, and consequently $c > 1$. Thus, we obtain:
\begin{equation}
    |\mathcal{V}(p_\psi, k)| = \Omega(c^k). 
\end{equation}
This finishes the proof. 

\end{proof}

Alternatively, we can consider the information capacity of the reasoning process:
\begin{proof}
(Information-theoretic perspective.) Let the generation of a natural language token sequence be modeled as a stochastic process $\mathcal{X} = \{X_1, X_2, \dots\}$. The uncertainty of this process is measured by its Entropy Rate, denoted as $H(\mathcal{X})$. 

While the theoretical maximum entropy of selecting from a vocabulary $\mathcal{A}$ is $\log_2 |\mathcal{A}|$, natural language is constrained by grammar and logic, resulting in a lower actual entropy. We define the Effective Vocabulary Size, $V_{eff}$, as the weighted branching factor of the language:
\begin{equation}
    V_{eff} = 2^{H(\mathcal{X})}. 
\end{equation}
Physically, $V_{eff}$ represents the average number of plausible next-token choices available to the model at any given step. For any expressive language, $V_{eff} > 1$. 

According to the Asymptotic Equipartition Property (AEP) in information theory~\cite{cover1999elements}, as the sequence length $k$ increases, the set of valid natural language sequences concentrates in the Typical Set $A_\epsilon^{(k)}$. The size of this set scales as: 
\begin{equation}
    |A_\epsilon^{(k)}| \approx 2^{k \cdot H(\mathcal{X})} = (V_{eff})^k. 
\end{equation}
Now, consider the subset of sequences that specifically encode the semantic concept $\psi$. Let $K(\psi)$ denote the minimal information (in bits) required to specify the semantics of $\psi$. The redundancy available for syntactic variation is the difference between the channel capacity and the semantic payload: 
\begin{equation}
    R(k) = k \log_2 V_{eff} - K(\psi). 
\end{equation}

This ``free information'' $R(k)$ allows for variations in style, phrasing, and intermediate derivation steps without altering the final semantic truth. The number of valid realizations $N$ is determined by this redundancy: 

\begin{equation}
    N \propto 2^{R(k)} = 2^{k \log_2 V_{eff} - K(\psi)} = \frac{(V_{eff})^k}{2^{K(\psi)}}. 
\end{equation}

Since $K(\psi)$ is a constant intrinsic to the concept, and $V_{eff} > 1$ for any non-trivial language model, the term $(V_{eff})^k$ dominates. Consequently, the number of semantically equivalent realizations grows exponentially with the sequence length $k$, i.e.: 
\begin{equation}
    |\mathcal{V}(p_\psi, k)| = \Omega(c^k). 
\end{equation}
This finishes the proof. 
\end{proof}

By simply applying Lemma~\ref{lm:exponetial-growth-of-semantic-equivalence-set}, we can directly derive Assumption~\ref{asm:exponential-expressive-redundancy}.

\subsection{Proof of Lemma~\ref{lm:continuum-convergence-of-reasoning-level-state-space}}
\label{app:proof-of-continuum-convergence-of-reasoning-level-state-space}
\begin{proof}
We begin with the definition of the number of realizations for a complete reasoning process. 
\begin{definition}\label{def:number-of-realizations}
(Number of realizations.) For a complete reasoning trajectory of $L$ steps, $(\xi_1, \xi_2, ..., \xi_L)$, its total number of token-level realizations, $N_{\text{realize}}$, is the size of the Cartesian product of the semantic equivalence sets for each step.
$$
    N_{\text{realize}}(\xi_1, ..., \xi_L) = \prod_{l=1}^{L} |\mathcal{V}(\xi_l,k_l)|. 
$$
\end{definition}

Consider a single, fixed reasoning-level path $(\xi_1, ..., \xi_L)$, where each step $\xi_l$ is realized by an average of $k$ tokens. By Assumption~\ref{asm:exponential-expressive-redundancy}, the number of token-level realizations for this single path, $N_{\text{realize}}^K$, is bounded below:
\begin{equation}\label{eq:bound-for-N-realize}
    N_{\text{realize}}^K \ge \prod_{l=1}^{L} c^{k_l} = c^{\sum_{l=1}^{L}k_l} = c^{kL} = c^{K}. 
\end{equation}
This shows that the number of token-level trajectories corresponding to even a single reasoning path grows exponentially with the total number of tokens, denoted as $K = kL$.

Next, we consider the total number of distinct, valid reasoning states, $|\mathcal{S}_{\text{reasoning}}^{(K)}|$. 
%For complex problems, the number of distinct valid reasoning paths itself grows with $K$. As each additional reasoning step can potentially branch from multiple previous states, the number of reachable states grows super-linearly, often exponentially, with $K$. 
For complex problems, given a total token amount $K$, the valid reasoning paths can be realized in distinct lengths that are no longer than $K$, which means $|\mathcal{S}_{\text{reasoning}}^{(K)}| = \sum_{K^{'} \le K}N_{\text{realize}}^{K^{'}}$.
According to Equation~(\ref{eq:bound-for-N-realize}), the number of realizations is bounded by an exponential term of the total number of tokens $K$, thus $|\mathcal{S}_{\text{reasoning}}^{(K)}|$ is of the same order with $N_{\text{realize}}^K$. 
\begin{equation}
    |\mathcal{S}_{\text{reasoning}}^{(K)}| = \Theta(c^{K}). 
\end{equation}

%This implies $|\mathcal{S}_{\text{reasoning}}^{(K)}| \propto \alpha^K$ for some base $\alpha > 1$. 
From Definition~\ref{def:reasoning-state-density}, the reasoning state density $\rho(K)$ is thus:
\begin{equation}
    \rho(K) = \frac{|\mathcal{S}_{\text{reasoning}}^{(K)}|}{|\mathcal{V}_{\text{semantic}}|} =\Theta(c^{K}). %\propto \frac{\alpha^K}{|\mathcal{V}_{\text{semantic}}|}. 
\end{equation}
Assuming the semantic volume $\mathcal{V}_{\text{semantic}}$ for a given problem domain is fixed, the state density $\rho(K)$ thus grows exponentially with $K$.

Finally, we connect high density to the expected distance between adjacent points with the following lemma. 

\begin{lemma}\label{lm:nearest-neighbor-distance-and-density}
(Nearest neighbor distance). Let points be distributed in a $d$-dimensional Euclidean space $\mathbb{R}^d$ according to a homogeneous Poisson point process with density $\rho$. The expected distance to the $k$-th nearest neighbor, denoted by $\mathbb{E}[R_k]$, is given by
$$ \mathbb{E}[R_k] = \frac{1}{(k-1)!} \left( \frac{1}{\rho C_d} \right)^{1/d} \Gamma\left(k + \frac{1}{d}\right), $$
where $C_d = \frac{\pi^{d/2}}{\Gamma(\frac{d}{2} + 1)}$ is the volume of a unit $d$-dimensional ball and $\Gamma(\cdot)$ is the Gamma function.
\end{lemma}
The proof is provided in Appendix~\ref{app:proof-of-nearest-neighbor-distance-and-density}. Lemma~\ref{lm:nearest-neighbor-distance-and-density} implies the relation between the density and the expected distance of the points in a space. By only considering the nearest adjacent points, in a $D$-dimensional space, the expected distance $\mathbb{E}[\operatorname{dist}_{s_i \bowtie s_j}(s_i, s_j)]$ between adjacent nearest neighbor points $s_i, s_j$ is inversely related to the density, approximately scaling as: 
\begin{equation}
    \mathbb{E}_{s_i \bowtie s_j}[\operatorname{dist}(s_i,s_j)] \propto \rho^{-1/D}. 
\end{equation}

%In a $D$-dimensional space, the expected distance $\mathbb{E}[d(s_i, s_j)]$ between adjacent points $s_i, s_j$ is inversely related to the density, approximately scaling as $\mathbb{E}[d(s_i, s_j)] \propto \rho(K)^{-1/D}$. 

Since $\rho(K)$ grows exponentially, it follows that:
\begin{equation}
    \lim_{K \to \infty} \mathbb{E}_{s_i \bowtie s_j}[\operatorname{dist}(s_i, s_j)] = 0. 
\end{equation}
When the expected distance between states approaches zero, it implies that for any state and any direction, another valid state can be found within an infinitesimally small neighborhood. This is the defining characteristic of a continuous space or manifold. Therefore, as the reasoning token amount $K$ increases, the discrete CoT-Space increasingly resembles a continuum, which validates the application of continuous optimization tools for its analysis. 

This finishes the proof. 
\end{proof}

\subsection{Proof of Lemma~\ref{lm:nearest-neighbor-distance-and-density}}
\label{app:proof-of-nearest-neighbor-distance-and-density}
\begin{proof}

The proof is built upon~\citet{haenggi2005distances} and ~\citet{miyagawa2018nearest}. 
%\textbf{1. The Probability Density Function (PDF)}

First, we derive the probability density function, $f_k(r)$, for the distance to the $k$-th nearest neighbor. The cumulative distribution function (CDF), $F_k(r)$, is the probability that a $d$-dimensional ball $B_r$ of radius $r$ contains at least $k$ points. The volume of this ball is $V_d(r) = C_d r^d$.

The number of points $X$ in a region of volume $V$ follows a Poisson distribution with mean $\rho V$: 
\begin{equation}
P(X=x) = \frac{(\rho V)^x}{x!} e^{-\rho V}.  
\end{equation}

The CDF is therefore $1$ minus the probability of finding fewer than $k$ points in the ball $B_r$:
\begin{equation}
F_k(r) = P(\text{points in } B_r \ge k) = 1 - \sum_{x=0}^{k-1} P(X=x) = 1 - e^{-\rho C_d r^d} \sum_{x=0}^{k-1} \frac{(\rho C_d r^d)^x}{x!}. 
\end{equation}

The PDF is the derivative of the CDF with respect to $r$:
\begin{equation}
f_k(r) = \frac{\dd}{\dd r} F_k(r) = - \frac{\dd}{\dd r} \left( e^{-\rho C_d r^d} \sum_{x=0}^{k-1} \frac{(\rho C_d r^d)^x}{x!} \right). 
\end{equation}

Using the product rule, the derivative simplifies due to a telescoping sum, yielding:
\begin{equation}
f_k(r) = \frac{d(\rho C_d)^k r^{dk-1}}{(k-1)!} e^{-\rho C_d r^d}. 
\end{equation}

%\textbf{2. The Expected Value}

The expected value $\mathbb{E}[R_k]$ is defined as the integral of $r$ times its PDF over its domain:
\begin{equation}
\mathbb{E}[R_k] = \int_0^\infty r \cdot f_k(r) \, \dd r. 
\end{equation}

Substituting the PDF we derived:
\begin{equation}
\mathbb{E}[R_k] = \int_0^\infty r \cdot \left( \frac{d(\rho C_d)^k r^{dk-1}}{(k-1)!} e^{-\rho C_d r^d} \right) \, \dd r. 
\end{equation}

We combine the terms involving $r$ and move the constants outside the integral:
\begin{equation}
\mathbb{E}[R_k] = \frac{d(\rho C_d)^k}{(k-1)!} \int_0^\infty r^{dk} e^{-\rho C_d r^d} \, \dd r. 
\end{equation}

To solve the integral, we perform a change of variables. Let $t = \rho C_d r^d$. This implies:
\begin{equation}
r = \left(\frac{t}{\rho C_d}\right)^{1/d} \quad \text{and} \quad \dd t = d \rho C_d r^{d-1} \, \dd r. 
\end{equation}

The integral becomes:
\begin{equation}
\begin{aligned}
\int_0^\infty r^{dk} e^{-\rho C_d r^d} \, \dd r &= \int_0^\infty \left(\frac{t}{\rho C_d}\right)^k e^{-t} \frac{\dd t}{d \rho C_d r^{d-1}} \\
&= \int_0^\infty \left(\frac{t}{\rho C_d}\right)^k e^{-t} \frac{\dd t}{d \rho C_d (t/(\rho C_d))^{(d-1)/d}} \\
&= \frac{1}{d(\rho C_d)^{k+1-(d-1)/d}} \int_0^\infty t^{k - (d-1)/d} e^{-t} \, \dd t \\
&= \frac{1}{d(\rho C_d)^{k+1/d}} \int_0^\infty t^{(k+1/d)-1} e^{-t} \, \dd t. 
\end{aligned}
\end{equation}

The remaining integral is the definition of the Gamma function, $\Gamma(k + 1/d)$. Thus:
\begin{equation}
\int_0^\infty r^{dk} e^{-\rho C_d r^d} \, \dd r = \frac{\Gamma(k + 1/d)}{d(\rho C_d)^{k+1/d}}. 
\end{equation}

Substituting this result back into the expression for $\mathbb{E}[R_k]$:
\begin{equation}
\mathbb{E}[R_k] = \frac{d(\rho C_d)^k}{(k-1)!} \left( \frac{\Gamma(k + 1/d)}{d(\rho C_d)^{k+1/d}} \right). 
\end{equation}

The terms $d$ and $(\rho C_d)^k$ cancel out, leaving:
\begin{equation}
\mathbb{E}[R_k] = \frac{1}{(k-1)!} \frac{\Gamma(k+1/d)}{(\rho C_d)^{1/d}}. 
\end{equation}

Which can be written as the final result:
\begin{equation}
\mathbb{E}[R_k] = \frac{1}{(k-1)!} \left( \frac{1}{\rho C_d} \right)^{1/d} \Gamma\left(k + \frac{1}{d}\right). 
\end{equation}

This finishes the proof.
\end{proof}

\subsection{Proof of Theorem~\ref{thm:convergence-of-continuum-errors}}
\label{app:proof-of-convergence-of-continuum-errors}

Our goal is to show that both error components converge to zero as the reasoning state density $\rho(K)$ increases. 
Consider a $D$-dimensional semantic space. 
This analysis is predicated on a reasonable assumption about the locality of a single reasoning step. 

% \begin{assumption}\label{asm:bounded-step-radius}
% (Bounded step radius). Any single reasoning step corresponds to a transition in the semantic space with an effective maximum radius $R_\text{max}$ and a minimum radius $R_\text{min}$. Thus, from a state $s$, the set of all possible subsequent states is contained within a $D$-dimensional hypersphere of radius $R$, denoted as $B_R(s)$. 
% \end{assumption} 

\begin{assumption}\label{asm:bounded-step-radius}
(Bounded step optimization). Any single reasoning step corresponds to a transition in the semantic space with an effective maximum radius $R_\text{max}$ and a minimum radius $R_\text{min}$. Thus, from a state $s$, the set of all possible subsequent states is contained within the region between two concentric $D$-dimensional hyperspheres centered at $s$ with radii $R_\text{min}$ and $R_\text{max}$. This region can be denoted as $B_R(s)$. 
\end{assumption}

Following Assumption~\ref{asm:bounded-step-radius}, we continue to use the following definition to characterize the distribution structure of state points in the semantic space. 
\begin{definition}\label{def:maximal-void}
    (Maximal void.) The maximal void $B_R^{\text{hole}}(s)$ is defined as the largest void hypersphere region within the bounded step radius. Formally, it is the largest hypersphere subspace of $B_R(s)$ that does not contain any state points: 
    $$B_{R}^{hole}(s) = \underset{B}{\arg\max} \, \left\{ r(B) \mid B \subseteq B_{R}(s)~\text{and}~B \cap \mathcal{S} = \emptyset \right\}, $$
    where $r(\cdot)$ is the radius of a ball and $\mathcal{S}$ is the reasoning state space. 
\end{definition}

Let $R_\text{hole}$ be the radius of $B_R^{\text{hole}}(s)$. We further show the relationship between $R_\text{hole}$ and the reasoning state density $\rho$. 
\begin{lemma}\label{lm:void-radius-and-density}
    (Maximal void radius and reasoning state density.) In a $D$-dimensional semantic space, the radius of the maximal void $R_\text{hole}$ and the density of reasoning state points $\rho(K)$ have the following relationship: 
    $$R_\text{hole} = \mathcal{O}\left(\rho(K)^{-\frac{1}{D}}\right).$$ 
\end{lemma}
\begin{proof}
The proof relies on the connection between the average nearest-neighbor distance and the scale of the largest voids in a spatial point process.

As established in Lemma~\ref{lm:continuum-convergence-of-reasoning-level-state-space}, which builds upon the theory of spatial point processes (Lemma~\ref{lm:nearest-neighbor-distance-and-density}), the expected distance between adjacent reasoning states in a $D$-dimensional semantic space is inversely related to the state density $\rho(K)$:
\begin{equation}
\mathbb{E}_{s_i \bowtie s_j}[\text{dist}(s_i, s_j)] \propto \rho(K)^{-1/D}. 
\end{equation}

The radius of the maximal void, $R_{\text{hole}}$, represents the largest possible distance from a point within that void to the nearest available reasoning state. It can thus be understood as a measure of the worst-case nearest-neighbor distance in a local region.

In the study of stochastic geometry, for a homogeneous point process, the distribution of nearest-neighbor distances is concentrated around its mean. Consequently, the scaling of the maximal values is of the same order as the scaling of the expected value.

Given that the expected distance is proportional to $\rho(K)^{-1/D}$, it follows that the maximal void radius, representing the worst-case local distance, must adhere to the same scaling law. Therefore, we can conclude that:
\begin{equation}
R_{\text{hole}} = \mathcal{O}\left(\rho(K)^{-\frac{1}{D}}\right). 
\end{equation}
This finishes the proof. 
\end{proof}

With Lemma~\ref{lm:void-radius-and-density}, we then commence to analyze the convergence of the two continuum errors, respectively. 

\textbf{Analysis of the Angular Error $\epsilon_A(s)$}

The angular error measures the difference in direction between discrete steps and continuous steps. Building upon the previous discussion, we can bound the angular error by the following result.

\begin{lemma}\label{lm:convergence-of-angular-error}
(Convergence of the Angular Error). The expected angular error $\mathbb{E}\left[\epsilon_A(s)\right]$ is bounded above by a function that is inversely related to the reasoning state density $\rho(K)$: 
$$ \mathbb{E}[\epsilon_A(s)] \le \mathcal{O}\left(\rho(K)^{-\frac{1}{D}}\right). $$
\end{lemma}
\begin{proof}
The angular error $\epsilon_{A}(s)$ is defined as the angle between the ideal step vector $\vec{v}_{\text{ideal}}$ and the best achievable discrete step vector $\vec{v}_{\text{real}}$. This error is determined by the directional sparsity of available states in the semantic space.

The worst-case error occurs when the ideal vector $\vec{v}_{\text{ideal}}$ points directly towards the center of the maximal void, $B_{R}^{\text{hole}}(s)$, within the local neighborhood $B_{R}(s)$. Let the radius of this maximal void be $R_{\text{hole}}$. The ideal target state $p_{\text{ideal}} = s + \vec{v}_{\text{ideal}}$ is at the center of this void.

By definition, the best achievable subsequent state, $s_{j}^{*}$, must lie on or outside the boundary of this void. The vector difference $s_{j}^{*} - p_{\text{ideal}}$ represents the deviation from the ideal target.

Let's consider the triangle formed by the points $s$ (current state), $p_{\text{ideal}}$ (ideal target), and $s_{j}^{*}$ (best achievable state). The angular error $\epsilon_{A}(s)$ is the angle at vertex $s$. Using the small-angle approximation for sine, which is valid as the space becomes dense, we can bound the angle: 
\begin{equation}
\epsilon_{A}(s) \approx \sin(\epsilon_{A}(s)) \approx \frac{\|s_{j}^{*} - p_{\text{ideal}}\|}{\|\vec{v}_{\text{real}}\|} \le \frac{\|s_{j}^{*} - p_{\text{ideal}}\|}{R_\text{min}}. 
\end{equation}

In the worst-case scenario, the distance from the ideal point to the best real point $\|s_{j}^{*} - p_{\text{ideal}}\|$ is bounded by the properties of the void. This distance is on the order of the void's radius, $R_{\text{hole}}$. Since the step size $\|\vec{v}_{\text{real}}\|$ is bounded and non-zero, the angular error is directly proportional to the radius of the maximal void: 
\begin{equation}
\epsilon_{A}(s) = \mathcal{O}(R_{\text{hole}}). 
\end{equation}

Taking the expectation over all possible configurations of the state space, we have: 
\begin{equation}
\mathbb{E}[\epsilon_{A}(s)] = \mathcal{O}(\mathbb{E}[R_{\text{hole}}]). 
\end{equation}

%From the field of stochastic geometry, it is a known result that for a D-dimensional space with a homogeneous point process of density $\rho$, the expected radius of the largest void scales as $\mathbb{E}[R_{\text{hole}}] \propto \rho^{-1/D}$. 

% Applying this to our reasoning state density $\rho(K)$:
% $$ \mathbb{E}[R_{\text{hole}}] \propto \rho(K)^{-1/D} $$
By combining this with Lemma~\ref{lm:void-radius-and-density}, we establish the bound on the expected angular error: 
\begin{equation}
\mathbb{E}[\epsilon_{A}(s)] \le \mathcal{O}(\rho(K)^{-\frac{1}{D}}). 
\end{equation}

As established in Lemma~\ref{lm:continuum-convergence-of-reasoning-level-state-space}, the reasoning state density $\rho(K)$ grows exponentially with the total reasoning token amount $K$. Consequently, as $K \rightarrow \infty$, $\rho(K) \rightarrow \infty$, and thus $\lim_{K\rightarrow\infty} \mathbb{E}[\epsilon_{A}(s)] = 0$. This demonstrates that the direction of the best discrete step converges to the direction of the ideal continuous gradient.

This finishes the proof. 

%The angular error is determined by the directional sparsity of discrete states. The worst-case error occurs when $\vec{v}_{\text {ideal }}$ points towards the center of the maximal void within the local neighborhood $B_R(s)$. Let the radius of this maximal void be $R_\text{hole}$. Let its center be at a distance $r$ from $s$. The angular radius of this void, which bounds $\epsilon_A(s)$, can be approximated for small angles as $\epsilon_A(s) \leq \sup \left(R_{\text {hole }} / r\right)$. This leads to the following theorem regarding the convergence of the angular error. 

%The proof follows from the result in stochastic geometry that the expected radius of the largest void, $\mathbb{E}\left[R_\text{hole}\right]$, scales as $\rho^{-\frac{1}{D}}$. Since the angular error is proportional to this radius, it inherits the same scaling property and thus converges to zero as $K \rightarrow \infty$. 

\end{proof}

\textbf{Analysis of the Magnitude Error $\epsilon_M(s)$}

The magnitude error quantifies how well the length of the discrete step matches the length of the ideal step. This insight provides a direct way to bound this error. 

\begin{lemma}\label{lm:convergence-of-magnitude-error}
(Convergence of the Magnitude Error). The expected magnitude error $\mathbb{E}\left[\epsilon_M(s)\right]$ is also bounded above by a function that is inversely related to the reasoning state density $\rho(K)$:
$$ \mathbb{E}[\epsilon_M(s)] \le \mathcal{O}\left(\rho(K)^{-\frac{1}{D}}\right). $$
\end{lemma}
\begin{proof}
By the reverse triangle inequality, the magnitude error is bounded by the norm of the vector difference: 
\begin{equation}
\epsilon_M(s) = \left| \| \vec{v}_{\text{real}} \| - \| \vec{v}_{\text{ideal}} \| \right| \le \| \vec{v}_{\text{real}} - \vec{v}_{\text{ideal}} \|. 
\end{equation}

Let $p_{\text {ideal }}=s+\vec{v}_{\text {ideal }}$ be the ideal target state. The best achievable state $s_j^*$ is the valid state closest to $p_{\text {ideal }}$. The vector difference is then $\left\|\left(s_j^*-s\right)-\left(p_{\text {ideal }}-s\right)\right\|=\left\|s_j^*-p_{\text {ideal }}\right\|$. 

In the worst-case scenario, the ideal target $p_{\text {ideal }}$ falls within the largest void. The closest available state $s_j^*$ must lie on or outside the boundary of this void. The distance between any point inside a hypersphere and its boundary is at most its diameter, $2R_\text{hole}$. Therefore, the distance between the ideal target and the best real state is bounded: 
\begin{equation}
\|s_j^* - p_{\text{ideal}}\| \le 2R_{\text{hole}}. 
\end{equation}

Combining these steps, we get a bound for the magnitude error: 
\begin{equation}
\epsilon_M(s) \le 2R_{\text{hole}}. 
\end{equation}

Taking the expectation and applying Lemma~\ref{lm:void-radius-and-density}, we find that the expected magnitude error is bounded by the same scaling law as the angular error: 
\begin{equation}
\mathbb{E}[\epsilon_M(s)] \le 2 \mathbb{E}[R_{\text{hole}}] = \mathcal{O}\left(\rho(K)^{-1/D}\right). 
\end{equation}

As $K \rightarrow \infty$, $\rho(K) \rightarrow \infty$, and thus $\lim _{K \rightarrow \infty} \mathbb{E}\left[\epsilon_M(s)\right]=0$. 

This finishes the proof. 
\end{proof}

% Lemmas~\ref{lm:convergence-of-angular-error} and~\ref{lm:convergence-of-magnitude-error} jointly provide the formal justification for treating the discrete reasoning-level state space as a continuous manifold for optimization. They demonstrate that the total error in approximating a true gradient vector with the best possible discrete step is bounded and converges to zero as the model's expressive capacity ($K$) increases. %This result licenses our subsequent use of continuous analytical tools, such as the reasoning loss descent formulation $\Delta s=-\frac{1}{L} \frac{d}{d s} \hat{C}$, 
% By ensuring that for a sufficiently $K$, a discrete step vector $\vec{v}_{\text {real }}$ that is an arbitrarily accurate approximation of the ideal vector $\vec{v}_{\text {ideal }}$ can always be found. 

\subsection{Proof of Theorem~\ref{thm:g-inverse-prop-L}}
\label{app:proof-of-g-inverse-prop-L}
\begin{proof}
As previously established, when viewed from a reasoning-level perspective, the reasoning space can be considered approximately continuous, i.e., $s_o \in \mathcal{S}$ where $\mathcal{S}$ is a continuous space. The state is then optimized continuously based on the policy and reasoning loss. This perspective enables us to analyze the reasoning process through a loss-reduction lens. 

For an $L$-step CoT reasoning process, this can be viewed as an $L$-step optimization based on the gradient of the reasoning loss. A single step of this optimization can be written as:

\begin{equation}\label{eq:reasoning-update}
    \Delta s = -\frac{1}{L}\frac{\mathrm{d}}{\mathrm{d}s}\hat{C}, 
\end{equation}
where $\hat{C}$ is a noisy estimation of the reasoning loss, influenced by the policy and random sampling. We assume that $\hat{C}$ is a random variable with expectation $C$. Here, $L$ represents the total reasoning steps, or the CoT length. 

%Thus the one-step optimization can also be written as: 
%\begin{equation}
%    \Delta s = -\frac{1}{L}\left(\frac{\mathrm{d}}{\mathrm{d}s}C+\left(\frac{\mathrm{d}}{\mathrm{d}s}\hat{C}-\frac{\mathrm{d}}{\mathrm{d}s}C\right)\right).
%\end{equation}
Despite the above gradient descent formulation, this optimization process can also be modeled as a stochastic differential equation (SDE) with $L$ steps: 
\begin{equation}\label{eq:reasoning-SDE}
    \frac{\mathrm{d}s}{\mathrm{d}t}=-\frac{\mathrm{d}C}{\mathrm{d}s}+\zeta(t),
\end{equation}
where $\zeta(t)$ is the noise term and $t = {1,\dots,L}$. %We assume the noise term $\zeta(t)$ is white noise with the statistical properties: $\langle\zeta(t)\rangle = 0$ and autocorrelation $\langle\zeta(t)\zeta(t')\rangle = gF(s)\delta(t-t')$, where $g$ is the noise scale, $F(s)$ s a function characterizing the magnitude of the reasoning loss estimation error at state $s$. 

Equations~(\ref{eq:reasoning-update}) and (\ref{eq:reasoning-SDE}) formulate the optimization process of CoT reasoning from the gradient update and SDE perspective, respectively. By combining the two equations related to noise, we can thus obtain the relationship between noise scale and CoT length, as shown in Theorem~\ref{thm:g-inverse-prop-L}. The detailed derivation of this result is provided below.

To bridge the discrete reasoning update with the continuous SDE model, we equate the variance induced by the noise term in both frameworks over a single reasoning step. 

First, we analyze the noise in the discrete one-step reasoning update from Equation (\ref{eq:reasoning-update}):
\begin{equation}
    \Delta s = -\frac{1}{L}\left(\frac{\mathrm{d}C}{\mathrm{d}s} + \left(\frac{\mathrm{d}\hat{C}}{\mathrm{d}s} - \frac{\mathrm{d}C}{\mathrm{d}s}\right)\right). 
\end{equation}
The noise originates from the estimation error, which we define as $\alpha = \left(\frac{\mathrm{d}\hat{C}}{ds} - \frac{\mathrm{d}C}{\mathrm{d}s}\right)$. We assume this error has a variance $\langle\alpha^2\rangle = F(s)$, where $F(s)$ is a function characterizing the magnitude of the reasoning loss estimation error at state $s$. The variance of the full noise term in the discrete update is therefore:
\begin{equation}\label{eq:update-variance}
    \left\langle \left(-\frac{1}{L}\alpha\right)^2 \right\rangle = \frac{1}{L^2}\langle\alpha^2\rangle = \frac{F(s)}{L^2}. 
\end{equation}

Next, we analyze the noise from the SDE model proposed in Equation (\ref{eq:reasoning-SDE}).

We assume the noise term $\zeta(t)$ is white noise with the statistical properties: $\langle\zeta(t)\rangle = 0$ and autocorrelation $\langle\zeta(t)\zeta(t')\rangle = \sigma F(s)\delta(t-t')$, where $g$ is the noise scale, $F(s)$ s a function characterizing the magnitude of the reasoning loss estimation error at state $s$, where $\sigma$ is the noise scale. To find the variance over a discrete step, we integrate the SDE over a time interval $\Delta t = 1/L$:
\begin{equation}
    \Delta s = \int_0^{1/L} \left(-\frac{\mathrm{d}C}{\mathrm{d}s} + \zeta(t)\right) \mathrm{d}t = -\frac{1}{L}\frac{\mathrm{d}C}{\mathrm{d}s} + \int_0^{1/L} \zeta(t) \mathrm{d}t. 
\end{equation}
The variance of the noise component of this update is:
\begin{equation}\label{eq:SDE-variance}
    \left\langle \left(\int_0^{1/L} \zeta(t) \mathrm{d}t\right)^2 \right\rangle = \int_0^{1/L} \int_0^{1/L} \langle\zeta(t)\zeta(t')\rangle \mathrm{d}t' \mathrm{d}t = \int_0^{1/L} \int_0^{1/L} \sigma F(s)\delta(t-t') \mathrm{d}t' \mathrm{d}t = \frac{\sigma F(s)}{L}. 
\end{equation}

By equating the variance from the discrete update (Eq. \ref{eq:update-variance}) and the SDE model (Eq. \ref{eq:SDE-variance}), we can solve for the noise scale $\sigma$:
\begin{equation}
    \frac{F(s)}{L^2} = \frac{\sigma F(s)}{L}. 
\end{equation}
Solving for $g$, we find:
\begin{equation}
    \sigma = \frac{1}{L}. 
\end{equation}
This result shows that the noise scale $g$ in our CoT-Space SDE model is inversely proportional to the CoT length $L$. 

This finishes the proof. 
\end{proof}

\subsection{Proof of Theorem~\ref{thm:noise-generalization-trade-off}}
\label{app:proof-of-noise-gen-trade-off}
%A core challenge in reasoning is avoiding overfitting to the superficial patterns of a specific prompt (i.e., prompt sensitivity), which leads to poor generalization on unseen queries. To quantify this, we introduce the concept of Thought Dispersion, inspired by the gradient dispersion in SGD analysis. 
\begin{proof}
We begin with the necessary assumption for our proof. 

\begin{assumption}
    (Simplified additive noise.) Let $\Delta_t = s_t - s_t^*$ be the deviation of the stochastic trajectory from the deterministic baseline at step $t$. The dynamics of this deviation are strictly given by
    \begin{equation}
        \Delta_t = \Delta_{t-1} - \eta [\nabla C(s_{t-1}, q) - \nabla C(s_{t-1}^*, q)] + \sigma \zeta_t.
    \end{equation}
    For the purpose of analyzing the raw exploration capacity (dispersion) provided by the noise term, we adopt the Simplified Additive Noise Assumption. We assume that for sufficiently small steps $\eta$ or within local regions where the gradient field is approximately constant (i.e., $\nabla C(s) \approx \nabla C(s^*)$), the contractive or expansive effects of the drift term difference are second-order compared to the diffusion term. Under this assumption, the deviation dynamics simplify to a random walk: 
    \begin{equation}
        \Delta_t \approx \Delta_{t-1} + \sigma \zeta_t. 
    \end{equation}
\end{assumption}

Consequently, the accumulated deviation at step $T$, denoted as $\Delta_T = s_T - s_T^*$, follows a centered Gaussian distribution with variance scaling linearly with the path length:
\begin{equation}
    \Delta_T \sim \mathcal{N}(0, T \sigma^2 I_d).
\end{equation}
This assumption allows us to explicitly characterize the ``search radius'' of the reasoning process as a function of the noise scale $\sigma$ and chain length $T$. 

We now state our main theorem, which bounds the expected population risk $R_{pop} \triangleq \mathbb{E}_{q, \zeta}[C(s_T, q)]$ in terms of the noise scale $\sigma$. 

% \begin{theorem}\label{thm:noise-generalization-trade-off}
%     (Noise-Generalization trade-off.) Assume the loss function $C(\cdot, q)$ is $\Gamma$-subguassian and $\beta$-smooth. For a reasoning process of length $T$, the expected population risk is upper-bounded by: 
%     \begin{equation}
%         R_{pop} \le \underbrace{R_{emp}^* + \frac{\beta T \sigma^2 d}{2}}_{\text{(I) Robustness Cost}} + \underbrace{\sqrt{\frac{2\Gamma^2}{n} \sum_{t=1}^{T} \frac{d}{2} \log \left( 1 + \frac{\eta^2 \mathbb{E}[\mathbb{V}_t(s_{t-1})]}{d \sigma^2} \right)}}_{\text{(II) Stability Gain}},
%     \end{equation}
%     where $N$ is the number of training examples (or few-shot exemplars), $R_{emp}^*$ is the empirical risk of the deterministic (greedy) path, and the expectation in Term (II) is taken over the reasoning trajectory.
% \end{theorem}

We decompose the population risk into the empirical risk $R_{emp}$ and the generalization gap: $R_{pop} \leq R_{emp} + |R_{pop} - R_{emp}|$, where $R_{emp}(s) \triangleq \frac{1}{n} \sum_{i=1}^n C(s, q_i)$. We analyze the impact of noise $\sigma$ on each term under the expectation form separately. 

\begin{equation}
    \underbrace{\mathbb{E}_{q, \zeta}[R_{pop}(s_T)]}_{\text{Total Expected Risk}} = \underbrace{\mathbb{E}_{q, \zeta}[R_{emp}(s_T)]}_{\text{Expected Empirical Risk}} + \underbrace{\mathbb{E}_{q, \zeta}[R_{pop}(s_T) - R_{emp}(s_T)]}_{\text{Expected Generalization Gap}}
\end{equation}

We first consider bounding the empirical risk $\mathbb{E}_{q, \zeta}[R_{emp}(s_T)]$. 

Let $s_T^*$ denote the deterministic reasoning outcome in the $T$-th step (where $\sigma=0$) and $s_T$ be the stochastic outcome. The accumulated noise $\Delta_T = s_T - s_T^*$ follows $\Delta_T \sim \mathcal{N}(0, T\sigma^2 I_d)$ under a simplified additive noise assumption. Performing a second-order Taylor expansion of the loss around $s_T^*$:
\begin{equation}
    \mathbb{E}[R_{emp}(s_T)] = \mathbb{E}[R_{emp}(s_T^*+\Delta_T)] \approx R_{emp}(s_T^*) + \mathbb{E}[\nabla R_{emp}(s_T^*)^\top \Delta_T] + \frac{1}{2} \mathbb{E}[\Delta_T^\top \mathbf{H}_{emp}(s_T^*) \Delta_T],
\end{equation}
where $\mathbf{H}_{emp}(s) = \nabla^2 R_{emp}(s) = \nabla^2 \left( \frac{1}{N} \sum_{i=1}^N C(s, q_i) \right)$.

Due to the zero mean of the noise, the first-order term vanishes to zero: 
\begin{equation}
    \mathbb{E}[\nabla R_{emp}(s_T^*)^\top \Delta_T] = 0. 
\end{equation}

For the second-order term, we have:
\begin{equation}
\begin{aligned}
\mathbb{E}[\Delta_T^\top \mathbf{H}_{emp} \Delta_T] &= \mathbb{E}[\text{Tr}(\Delta_T^\top \mathbf{H}_{emp} \Delta_T)] \\
&= \mathbb{E}[\text{Tr}(\mathbf{H}_{emp} \Delta_T \Delta_T^\top)] \\
&= \text{Tr}(\mathbf{H}_{emp} \cdot \mathbb{E}[\Delta_T \Delta_T^\top]). 
\end{aligned}
\end{equation}

Recall that $\Delta_T \sim \mathcal{N}(0, T\sigma^2 I_d)$, we have: 
\begin{equation}
    \mathbb{E}[\Delta_T \Delta_T^\top] = T \sigma^2 I_d. 
\end{equation}

By the definition, 
\begin{equation}
    \mathbf{H}_{emp}(s) = \nabla^2 R_{emp}(s) = \nabla^2 \left( \frac{1}{n} \sum_{i=1}^n C(s, q_i) \right) = \frac{1}{n} \sum_{i=1}^n \nabla^2 C(s, q_i) \triangleq \frac{1}{n} \sum_{i=1}^n \mathbf{H}_i(s). 
\end{equation}

Using the $\beta$-smoothness condition, $\text{Tr}(\mathbf{H}_i(s_T^*)) \le d \beta$, thus: 
\begin{equation}
\begin{aligned}
\text{Tr}(\mathbf{H}_{emp}) &= \text{Tr}\left( \frac{1}{n} \sum_{i=1}^n \mathbf{H}_i \right) 
= \frac{1}{n} \sum_{i=1}^n \text{Tr}(\mathbf{H}_i)) \\
&\le \frac{1}{n} \sum_{i=1}^n (d \beta) = \frac{1}{n} \cdot n \cdot (d \beta) = d \beta. 
\end{aligned}
\end{equation}

Thus, the expected empirical risk increases by at most $\frac{\beta T \sigma^2 d}{2}$. Let $T=L$, this constitutes Term (I): 
\begin{equation}\label{eq:upper-bound-of-emprical-risk}
    \mathbb{E}[R_{emp}(s_L)] \le R_{emp}(s_L^*) + \frac{\beta L \sigma^2 d}{2}. 
\end{equation}

This term represents the Optimization Drift: higher noise requires the solution to lie in a flatter region of the loss landscape to maintain accuracy. 

We subsequently consider bounding the generalization gap. 

Following the information-theoretic framework~\cite{xu2017information}, the generalization gap is bounded by the mutual information between the input query data $Q$ and the final reasoning state $s_L$: $\mathbb{E}|R_{pop} - R_{emp}| \le \sqrt{\frac{2\Gamma^2}{n} I(s_L; Q)}$.
By the data processing inequality and the chain rule, $I(s_T; Q) \le \sum_{t=1}^L I(s_t; Q | s_{t-1})$. We invoke Lemma 4 from \citet{wang2022on}, which provides a tight bound for additive Gaussian noise channels:

\begin{lemma}
    (Upper bound for additive Gaussian noise channels~\cite{wang2022on}.) Let random variables $X,Y$ and $\Delta$ be independent of $N \sim \mathcal{N}(0, I_d)$. Then for any $\sigma > 0$, any $\mathbb{R}^d$-valued function $f$, and any random variable $\Omega \in \mathbb{R}^d$ that is a function of $Y$, we have 
    \begin{equation}
        I(f(Y+\Delta, X)+\sigma N ; X \mid Y) \leq \frac{d}{2} \mathbb{E}\left[\log \left(\frac{\mathbb{E}\left[\|f(Y+\Delta, X)-\Omega\|^2\right]}{d \sigma^2}+1\right)\right].
    \end{equation}
\end{lemma}

For a step $s_t = s_{t-1} - \eta \nabla C(s_{t-1}, q) + \sigma \zeta_t$, by mapping the query data $Q$ to $X$, the previous state $s_{t-1}$ to $Y$, and the drift term $-\eta \nabla C(s_{t-1}, Q)$ to the function $f(Y, X)$, we obtain: 

\begin{equation}
\begin{aligned}
I(s_t; Q | s_{t-1}) &\le \frac{d}{2} \log \left( 1 + \frac{\mathbb{E}_{Q}[\| -\eta \nabla C(s_{t-1}, Q) - \mathbb{E}[-\eta \nabla C] \|^2]}{d \sigma^2} \right) \\
&= \frac{d}{2} \log \left( 1 + \frac{\eta^2 \mathbb{V}_t(s_{t-1})}{d \sigma^2} \right),
\end{aligned}
\end{equation}

Summing over $t=1 \dots L$ yields Term (II): 

\begin{equation}\label{eq:upper-bound-of-generalization-error}
    \mathbb{E}|R_{pop} - R_{emp}| \le \sqrt{\frac{2\Gamma^2}{n} \sum_{t=1}^{L} \frac{d}{2} \log \left( 1 + \frac{\eta^2 \mathbb{E}[\mathbb{V}_t(s_{t-1})]}{d \sigma^2} \right)}. 
\end{equation}

This term is a decreasing function of $\sigma$, reflecting that noise masks the specific details of the input prompt, thereby preventing the reasoning trajectory from overfitting to prompt-specific shortcuts. 

Combining Equations~(\ref{eq:upper-bound-of-emprical-risk}) and~(\ref{eq:upper-bound-of-generalization-error}), this finishes the proof. 

\end{proof}

\subsection{Proof and Analysis of Theorem~\ref{thm:mi-reasoning-gen-err-upbd}}
\label{app:analysis-for-mi-reasoning-error-upbd}
To derive an upper bound for the reasoning generalization error, we can first simply adapt the PAC-Bayes framework, which provides a powerful link between generalization, policy complexity, and data volume. For our analysis within CoT-Space, we establish the following components: 
\begin{itemize}
    \item \textbf{Hypothesis Space $\mathcal{H}$:} The space of all possible reasoning policies, denoted by $\mathcal{H}$. Our learned policy $\pi \in \mathcal{H}$.
    \item \textbf{Prior Distribution $P$:} A distribution over the policy space $\mathcal{H}$ reflecting our initial beliefs about which policies are "good," often favoring simpler ones. 
    \item \textbf{Posterior Distribution $Q$:} A distribution over $\mathcal{H}$ that is learned from the training data queries $\{q_i\}_{i=1}^n$, concentrating on policies that perform well on that data.
\end{itemize}

By applying the PAC-Bayes theorem~\citep{mcallester1998some}, we obtain a bound on the expected true loss which, when rearranged, yields an upper bound on the expected generalization error. 

\begin{theorem}\label{thm:pac-bayes-upbd}
(PAC-Bayesian generalization upper bound for LLM reasoning.) With CoT-Space defined above, the policy $\pi$ is trained on an i.i.d. drawn training set with size $n$,by assuming the reasoning loss is bounded in $[0, 1]$, for any prior $P$, with probability at least $1-\delta$, the following holds for all posterior $Q$: 
$$
    \mathbb{E}_{\pi\sim Q}[R(\pi) - \hat{R}_n(\pi)] \leq \sqrt{\frac{\operatorname{KL}(Q||P) + \ln\frac{2n}{\delta}}{2n}}, 
$$
where $\operatorname{KL}(Q||P) = \mathbb{E}_{\pi\sim Q}[\ln\frac{Q(\pi)}{P(\pi)}]$ is the Kullback-Leibler (KL) divergence between the posterior and prior distributions.
\end{theorem}
\begin{proof}
Let's denote the true reasoning loss of a policy $\pi$ as $R(\pi) = \mathbb{E}_{q\sim\mathcal{D}}[C(s_q^\pi,q)]$ and the empirical reasoning loss as $\hat{R}_n(\pi) = \frac{1}{n}\sum_{i=1}^n C(s_{q_i}^\pi,q_i)$. The reasoning generalization error for a specific policy $\pi$ is $R(\pi) - \hat{R}_n(\pi)$. We aim to bound this gap.

The PAC-Bayes theorem provides a bound on the expected true loss for policies drawn from the posterior $Q$. A standard version of the theorem is often presented as follows. 

\begin{lemma} \label{lm:pac-bayes-gen-upbd}
(PAC-Bayes generalization upper bound~\citep{mcallester1998some}.) For any loss function bounded in $[0, 1]$, for any prior $P$, with probability at least $1-\delta$ over the draw of the training set $\{q_i\}_{i=1}^n$, the following holds for all posterior distributions $Q$:
$$
    \mathbb{E}_{\pi\sim Q}[R(\pi)] \leq \mathbb{E}_{\pi\sim Q}[\hat{R}_n(\pi)] + \sqrt{\frac{KL(Q||P) + \ln\frac{2n}{\delta}}{2n}}, 
$$
where $\operatorname{KL}(Q||P) = \mathbb{E}_{\pi\sim Q}[\ln\frac{Q(\pi)}{P(\pi)}]$ is the Kullback-Leibler (KL) divergence between the posterior and prior distributions.
\end{lemma}

Let's adapt Lemma~\ref{lm:pac-bayes-gen-upbd} to our framework. We first need to normalize our reasoning loss $C(\cdot,\cdot)$ to be in $[0, 1]$. Assuming there is a maximum possible loss $C_{max}$, we can define a normalized loss $C'(\cdot,\cdot) = C(\cdot,\cdot)/C_{max}$. The bound then applies to the normalized losses $R'(\pi)$ and $\hat{R}'_n(\pi)$. For simplicity of notation, let's assume $C(\cdot,\cdot)$ is already normalized in $[0, 1]$. 

This finishes the proof. 
\end{proof}

%\paragraph{Information-Theoretic Interpretation}
The bound has a clear information-theoretic meaning, which aligns with the goal of understanding reasoning phenomena:

    \textbf{Complexity Term ($\operatorname{KL}(Q||P)$):} The KL divergence term measures the information gain in moving from our prior belief $P$ to our posterior belief $Q$ after seeing the data. It quantifies the complexity of the learned policy distribution. If $Q$ is very different from $P$ (i.e., $\operatorname{KL}(Q||P)$ is large), it means we had to learn a great deal from the training data to find a good policy. This is analogous to overfitting or memorizing the training set. The bound becomes looser (larger), penalizing this complexity and suggesting a higher risk of poor generalization. If $Q$ remains close to $P$ ($\operatorname{KL}(Q||P)$ is small), it implies the learned policy did not stray far from our initial bias towards simpler policies. The bound is tighter, suggesting better generalization. 

    \textbf{Data Term ($\frac{1}{n}$):} The bound depends on the number of training queries $n$. As $n$ increases, the term under the square root decreases, making the bound tighter. This confirms the intuition that having more data leads to better generalization and reduces the gap between training and test performance.

In conclusion, this PAC-Bayes bound formalizes the trade-off in learning a reasoning policy. To achieve good generalization, a policy $\pi$ must not only achieve a low empirical reasoning loss $\hat{R}_n(\pi)$ on the queries it has seen, but it must also remain simple in an information-theoretic sense, by not deviating too much from a reasonable prior. This framework provides a principled way to analyze and potentially control for phenomena like "overthinking" by viewing it as a form of overfitting, which would correspond to a large $\operatorname{KL}(Q||P)$ term. 

Theorem \ref{thm:pac-bayes-upbd} provides a powerful, high-probability upper bound on the expected reasoning generalization error for policies sampled from the learned posterior $Q$. While this PAC-Bayesian bound is a valid, generic result, we can derive a more specific bound by using information theory, a perspective that connects the generalization error to the mutual information between the learned policy and the training data. This approach leads to the information-theoretic upper bound in Theorem~\ref{thm:mi-reasoning-gen-err-upbd}. 

We continue to provide the proof of Theorem~\ref{thm:mi-reasoning-gen-err-upbd} here. 

\begin{proof}
The core idea is that a policy can only overfit to the training set $S = \{q_1, \dots, q_n\}$ by encoding information about $S$ within its own parameters. The mutual information $I(\pi; S)$ precisely quantifies this amount of encoded information. A key result from learning theory provides the following bound on the expected generalization error:

\begin{lemma}\label{lm:mi-gen-upbd-base}
(Information-theoretical generalization upper bound~\citep{xu2017information}.) With CoT-Space defined above, if the policy $\pi$ is trained on an i.i.d. drawn training set $S$ with size $n$, the following upper bound for the reasoning generalization error holds:
$$
    |\mathbb{E}[R(\pi) - \hat{R}_n(\pi)]| \leq \sqrt{\frac{C_{max}^2 \cdot I(\pi; S)}{2n}}, 
$$
where $C_{max}$ is the maximum possible value of the reasoning loss $C(\cdot)$, and the expectation is over the random draw of the training set $S$ and the (potentially stochastic) output $\pi$ of the learning algorithm. 
\end{lemma}

Lemma~\ref{lm:mi-gen-upbd-base} transforms the problem of bounding generalization error into the problem of bounding the mutual information $I(\pi; S)$. To make this bound meaningful for our framework, we must connect $I(\pi; S)$ to the structural properties of the generated Chain-of-Thought.

We then turn to bound mutual information by CoT complexity. 
The policy $\pi$ expresses the information it has learned from $S$ by generating a CoT, which is a sequence of reasoning steps $\xi = (\xi_1, \xi_2, \dots, \xi_L)$. Each step $\xi_i$ is, in turn, a sequence of tokens. The total information capacity of the policy's output is therefore determined not just by the number of steps, but by the total number of tokens generated. 

Let us define:
\begin{itemize}
    \item $L$: The number of reasoning steps in a CoT.
    \item $|\xi_i|$: The number of tokens in the $i$-th reasoning step.
    \item $\sum_{i=1}^L |\xi_i|$: The total number of tokens in the CoT.
\end{itemize}

The mutual information $I(\pi; S)$ is upper-bounded by the entropy of the policy's output, which represents its description length or information capacity. The total number of token-level decisions the policy makes determines this capacity. Let $\mathbb{E}[L]$ be the expected number of reasoning steps (reasoning depth) and $\mathbb{E}[|\xi|]$ be the expected number of tokens per step (step verbosity). The expected total number of tokens is $\mathbb{E}[L] \cdot \mathbb{E}[|\xi|]$.

If we assume the policy generates tokens from a dictionary $\mathcal{A}$, the information required to specify one token is at most $\log|\mathcal{A}|$. Therefore, we can establish the following upper bound on the mutual information.

\begin{lemma}\label{lm:mi-upbd}
In LLM reasoning scenarios, if the policy $\pi$ is trained on a training set $S$, the mutual information $I(\pi;S)$ is upper bounded as follows: 
$$
    I(\pi; S) \le \mathbb{E}_{S, \pi}\left[\sum_{i=1}^L |\xi_i|\right] \cdot \log|\mathcal{A}| = \left(\mathbb{E}[L] \cdot \mathbb{E}[|\xi|]\right) \cdot \log|\mathcal{A}|. 
$$
\end{lemma}
Lemma~\ref{lm:mi-upbd} formalizes the intuition that the information a policy learns from data cannot exceed the information content it is capable of expressing in its output. 

Together with Lemma~\ref{lm:mi-gen-upbd-base}, this finishes the proof. 
\end{proof}

\subsection{Proof of Theorem~\ref{thm:empirical-risk-lwbd}}
\label{app:proof-for-empirical-risk-lwbd}
\begin{proof}

To formalize this, we quantify the notion of problem complexity and establish a condition for policy failure. 

\begin{definition}\label{def:required-reasoning-depth}(Required reasoning depth.) 
For any given query $q$, we define its \textit{Required Reasoning Depth}, denoted as $L^*(q)$, as the minimum number of CoT steps required to reach any success state $m \in M_q$ from the initial state $(q, \emptyset)$. This value is an intrinsic property of the query's complexity.
\end{definition}

Definition~\ref{def:required-reasoning-depth} can be likened to the minimum number of steps required to solve a reasoning problem. For instance, a simple two-digit addition problem might require only one or two mental steps, making its $L^{*}$ very low. In contrast, proving a complex theorem, such as the Pythagorean theorem, requires a sequence of many interconnected logical steps, resulting in a high $L^{*}$. A policy that attempts to solve the theorem with only a few steps, regardless of how smart it is, will inevitably fail. 

Building on the concept of $L^{*}$, we now formalize the direct consequence of a policy's generated reasoning being too shallow. We make a straightforward assumption about the relationship between insufficient reasoning depth and policy failure, which translates a structural property of the CoT into a quantifiable error. 

\begin{assumption} (Policy failure on insufficient depth)\label{asm:policy-failure-on-insufficient-depth}
When a policy $\pi$ generates a CoT of length $L_i$ for a query $q_i$, if the length is insufficient to solve the problem, i.e., $L_i < L^*(q_i)$, the policy is guaranteed to fail. The resulting reasoning loss $C(s_{q_i}^\pi)$ is therefore bounded below by a significant positive constant, which we denote as $C_{fail} > 0$. In the worst case, this loss can be as high as $C_{max}$.
\end{assumption}

Assumption~\ref{asm:policy-failure-on-insufficient-depth} claims that a policy's failure due to insufficient depth can be observed in practice. Consider a complex physics problem requiring multiple steps: identifying the forces, setting up equations, and solving for variables. A policy that only completes the first step and then stops will not be able to reach the correct final answer. In this case, the loss function would assign a high penalty (our $C_{fail}$) because the generated CoT, while possibly correct for the initial step, is fundamentally incomplete for solving the problem. The policy, therefore, cannot even fit the training data for problems that require deeper reasoning.

Building upon Definition~\ref{def:required-reasoning-depth} and Assumption~\ref{asm:policy-failure-on-insufficient-depth}, we can then derive a lower bound of the empirical reasoning loss $\hat{R}_n(\pi)$. 

We begin with its definition:
\begin{equation}
    \hat{R}_n(\pi) = \frac{1}{n}\sum_{i=1}^n C(s_{q_i}^\pi). 
\end{equation}
We can partition the sum over the training set into two disjoint subsets based on whether the generated CoT length was sufficient:
\begin{itemize}
    \item The set of failed queries: $S_{fail} = \{q_i \in S \mid L_i < L^*(q_i)\}$. 
    \item The set of potentially successful queries: $S_{succ} = \{q_i \in S \mid L_i \ge L^*(q_i)\}$. 
\end{itemize}
Let $n_{fail} = |S_{fail}|$. We can now rewrite the empirical loss as:
\begin{equation}
    \hat{R}_n(\pi) = \frac{1}{n} \left( \sum_{q_i \in S_{fail}} C(s_{q_i}^\pi) + \sum_{q_i \in S_{succ}} C(s_{q_i}^\pi) \right). 
\end{equation}
Applying our failure assumption, we know that for every $q_i \in S_{fail}$, $C(s_{q_i}^\pi) \ge C_{fail}$. For queries in $S_{succ}$, the loss is at least $0$. This allows us to establish a lower bound:
\begin{equation}
    \hat{R}_n(\pi) \ge \frac{1}{n} \left( \sum_{q_i \in S_{fail}} C_{fail} + \sum_{q_i \in S_{succ}} 0 \right) = \frac{n_{fail} \cdot C_{fail}}{n}. 
\end{equation}
We can define the empirical failure rate of the policy $\pi$ on the training set as $P_\pi(L < L^*) = n_{fail}/n$. This leads to our final lower bound for the empirical loss:
\begin{equation}
    \hat{R}_n(\pi) \ge P_\pi(L < L^*) \cdot C_{fail}. 
\end{equation}
This finishes the proof. 
\end{proof}
\section{Experimental Settings}
\label{app:experimental-settings}
\textbf{Dataset}   The experiments are conducted on the MATH~\citep{hendrycks2021measuring} and GSM8K~\citep{cobbe2021training} benchmarks.  The MATH training set is manually subdivided into five difficulty levels (Level 1–5) according to the original ``level'' field, where Level 5 is substantially more challenging than Level 1.

\textbf{Base Model}   A diverse collection of pre-trained LLMs with varying architectures, parameter counts, and versions serves as the reference model for the reinforcement-learning algorithms. Specifically, the models include
\begin{itemize}
    \item \textbf{Qwen2.5 series}~\citep{qwen2025qwen25technicalreport}: Qwen2.5-7B, Qwen2.5-3B, Qwen2.5-1.5B, and Qwen2.5-0.6B. 
    \item \textbf{Qwen3 series}~\citep{yang2025qwen3technicalreport}: Qwen3-8B-Base, Qwen3-4B-Base, Qwen3-1.7B-Base, and Qwen3-0.6B-Base. 
    \item \textbf{Llama3 series}~\citep{grattafiori2024llama3herdmodels}: Llama-3.1-8B-Instruct, Llama-3.2-3B-Instruct, and Llama-3.2-1B-Instruct. 
\end{itemize}

The maximum response length is set to 2,048 tokens for the Qwen2.5 and Llama3 models and to 8,196 tokens for the reasoning-intensive Qwen3 models.

\textbf{Algorithm and Hyperparameters}   The evaluation covers the current mainstream reinforcement-learning algorithms for mathematical reasoning: GRPO~\citep{shao2024deepseekmath}, PPO~\citep{schulman2017proximal}, DAPO~\citep{yu2025dapo}, Reinforce++~\citep{hu2025reinforce++}, and RLOO~\citep{ahmadian2024back}. The reward function is rule-based~\citep{deepseekai2025deepseekr1incentivizingreasoningcapability}. 
For PPO, the critic model is initialized from Qwen2.5-0.6B. The learning rate for critic training is $1 \times 10^{-5}$, the k3 KL-regularization coefficient is 0.001, and the GAE parameters $\lambda$ and  $\gamma$ are both set to $1$. For DAPO, no KL regularization is applied; the high and low clip ratios are $0.28$ and $0.20$, respectively, and $L_{cache}$ is $1,024$.
Across all algorithms, each training sample is sampled eight times ($K = 8$). The actor is optimized with AdamW~\citep{loshchilov2017decoupled} at a learning rate of 1e-6, the training batch size is $128$, and training proceeds in a purely on-policy manner for up to 50 steps.  All experiments are run on 8 $\times$ 96 GB NVIDIA H20 GPUs. 
\section{More Experiments}
\label{app:more-experiments}

\begin{figure}[tp]
    \centering
    \includegraphics[width=\linewidth]{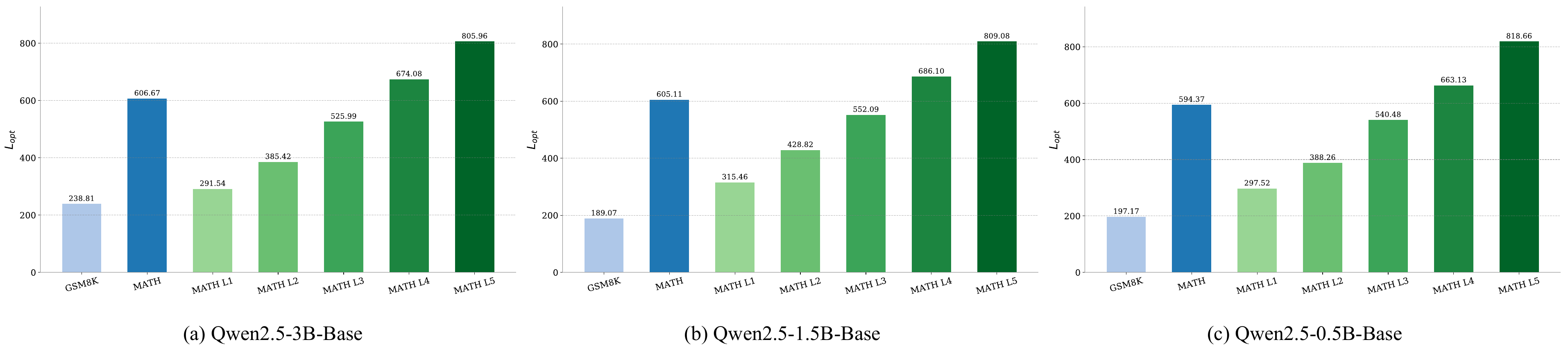}
    \caption{\textbf{Supplementary validation of the relationship between task difficulty and optimal CoT length ($L_\text{opt}$) on additional models.} The plots illustrate that the converged mean response length ($L_\text{opt}$) consistently increases with task difficulty across three different models: (a) Qwen2.5-3B-Base, (b) Qwen2.5-1.5B-Base, and (c) Qwen2.5-0.5B-Base. These results reinforce the findings presented in Figure~\ref{fig:experiment}(a) and demonstrate the robustness of Remark 1 across models of varying capacities.}
    \label{fig:more-experiment}
\end{figure}

To further strengthen the empirical validation of Remark 1, we conducted supplementary experiments to verify the relationship between task difficulty and the optimal CoT length ($L_\text{opt}$) across models of varying capacities. The experimental protocol is identical to that described in Section~\ref{sec:experiments}, but here we replicate the analysis on three additional models from the Qwen2.5 series: Qwen2.5-3B-Base, Qwen2.5-1.5B-Base, and Qwen2.5-0.5B-Base.

The results, presented in Figure~\ref{fig:more-experiment}, consistently corroborate the findings from our main experiment (Figure~\ref{fig:experiment}(a)). Across all three smaller models, we observe the same clear, monotonic trend: the converged optimal CoT length ($L_\text{opt}$) increases systematically as the difficulty of the task escalates from GSM8K to the higher levels of the MATH dataset. This demonstrates that the principle outlined in Remark 1, that more complex problems necessitate a longer reasoning process to avoid underfitting is not specific to a single high-capacity model but holds true as a general principle across models of different scales. 
\section{Broader Explanatory Power of CoT-Space}
\label{app:broader-exp-of-cot-space}

This section further explores the extensibility of the CoT-Space framework, demonstrating its capacity to provide a unified theoretical account for several other critical and complex phenomena observed in LLMs, beyond the ``overthinking'' problem analyzed in the main body of this paper.

\subsection{Hallucination as Trapped Optimization in Phantom Minima}
The phenomenon of hallucination~\citep{ji2023survey}, where a model generates fluent but factually incorrect statements, can be elegantly re-contextualized within our framework. In CoT-Space, reasoning is modeled as an optimization process on a continuous semantic manifold, aimed at minimizing a reasoning loss function $C(s)$. An ideal reasoning trajectory converges to a global minimum $M_q$, which corresponds to the correct solution. Hallucination can thus be explained as the optimization trajectory becoming trapped in a ``phantom minimum'' or a deep local minimum within the loss landscape. These states, while possessing a low reasoning loss due to their local coherence and syntactical correctness, do not correspond to the ground truth. Once a model's reasoning state $s_t$ enters the basin of attraction for such a trap, subsequent gradient-like updates will cause it to converge there, rather than to the true global minimum, yielding a chain-of-thought that is internally consistent but externally false.

\subsection{Prompt Sensitivity as a Consequence of Landscape Instability}
Furthermore, the acute sensitivity of LLMs to minor, semantically irrelevant perturbations in the input prompt~\citep{zhuo2024prosa,razavi2025benchmarking} can be understood as an initial state perturbation problem on an unstable loss landscape. The input query $q$ defines the initial state of the reasoning process, $s_0 = (q, \emptyset)$, which serves as the starting point for the optimization trajectory. The high-dimensional, non-convex loss landscape is likely replete with unstable topological features such as sharp ridges, crevasses, or watersheds. If an initial state $s_0$ is located near such an unstable region, a slight perturbation of the prompt (from $q$ to $q'$) can shift the initial state to $s'_0$. Although the distance between $s_0$ and $s'_0$ in the semantic space may be negligible, they might lie on opposite sides of a watershed. Consequently, their subsequent optimization trajectories will follow entirely different descent paths, converging to distinct minima and producing dramatically different outputs. This provides a formal, geometric interpretation for the model's apparent brittleness.

\subsection{Emergent Abilities as a Topological Phase Transition}
The framework also offers a compelling explanation for the ``emergent abilities'' of LLMs, where complex reasoning capabilities appear to manifest suddenly once model scale surpasses a certain threshold, a phenomenon often described as a ``reasoning cliff'' for smaller models~\citep{wei2022emergent}. This can be interpreted as a change in the topological properties of the reasoning space itself, which is directly influenced by model capacity. Smaller models, with their limited representational power, exhibit a lower reasoning state density $\rho(K)$, as defined in Definition~\ref{def:reasoning-state-density}. As established in Theorem~\ref{thm:convergence-of-continuum-errors}, a low density implies a larger expected distance between adjacent states, rendering the corresponding semantic manifold $\tilde{S}$ sparse, perforated, or even disconnected. For a complex problem, the path from the initial state $s_0$ to the solution $M_q$ may contain unbridgeable gaps, precluding any successful optimization trajectory and thus creating the reasoning cliff. As model scale increases, the representational capacity grows, leading to an exponential increase in the state density $\rho(K)$. This makes the manifold progressively more dense and smooth, effectively ``filling in the holes'' and establishing continuous pathways for reasoning. The emergence of an ability corresponds to a topological phase transition, where the manifold's connectivity reaches a critical point that permits a viable optimization process.

\subsection{The Efficacy of External Slow-Thinking}
Finally, CoT-Space provides a principled basis for the superior performance of external slow-thinking strategies like ToT~\citep{yao2023tree} or MCTS~\citep{wan2024alphazero} compared to standard CoT. The difference in their efficacy can be framed as the difference between two optimization algorithms operating on the same loss landscape. Standard CoT generation is analogous to a greedy, single-path gradient descent. At each step, it commits to a single, locally optimal thought, following a solitary trajectory that is highly susceptible to being trapped in the first local minimum it encounters. In contrast, external slow-thinking strategies maintain multiple parallel candidate reasoning paths. This is equivalent to performing a beam search within the semantic state space. By exploring multiple descent directions from a given state $s_t$, external slow-thinking strategies conduct a more comprehensive exploration of the complex, non-convex loss landscape. This strategy significantly enhances the probability of circumventing local traps and identifying a path to the global optimum, thereby explaining its enhanced effectiveness on complex reasoning tasks. 

\section{Prescriptive Implications}
\label{app:prescriptive-implications}
Beyond explaining current phenomena, CoT-Space provides a theoretical justification for Latent Space Reasoning. Since we prove that the reasoning space converges to a continuous manifold $\tilde{S}$, it is theoretically sound to map discrete CoT steps into this continuous latent space (e.g., via variational auto-encoders). Optimization (reasoning) could then be performed via gradient-based planning in this smooth latent space, avoiding the discrete, non-differentiable nature of token generation, before decoding the final result. Furthermore, our framework suggests that current sparse reward models are inefficient proxies for the smooth reasoning loss $C(s,q)$. Future work should focus on training Continuous Verifiers that predict a scalar ``semantic distance'' to the solution, providing the dense, smooth gradients required for more stable RL optimization. 

%%%%%%%%%%%%%%%%%%%%%%%%%%%%%%%%%%%%%%%%%%%%%%%%%%%%%%%%%%%%%%%%%%%%%%%%%%%%%%%
%%%%%%%%%%%%%%%%%%%%%%%%%%%%%%%%%%%%%%%%%%%%%%%%%%%%%%%%%%%%%%%%%%%%%%%%%%%%%%%

\end{document}